\documentclass[10pt]{article}
\PassOptionsToPackage{dvipsnames}{xcolor}
\usepackage{lafa}
\providecommand\MakeLinkTarget[2][]{}
\usepackage{microtype}
\usepackage{parskip}
\usepackage{booktabs}
\usepackage{caption}
\usepackage{multirow}
\usepackage{wrapfig}
\usepackage{makecell}
\usepackage{adjustbox}
\usepackage{fullpage}
\usepackage{ragged2e}

\usepackage[ruled,vlined]{algorithm2e}

\SetCommentSty{mycommfont}

\newtcolorbox{samplebox}[1]{
	enhanced,
	colback=LinkBlue!5!white,
	colframe=black!60,
	boxrule=0.6pt,
	arc=6pt,
	left=4pt, right=4pt, top=4pt, bottom=4pt,
	boxsep=3pt,
	before upper={\noindent\textbf{\small #1}\par\smallskip},
}

\usepackage{xcolor}

\usepackage{amsmath}
\usepackage{amssymb}
\usepackage{mathtools}
\usepackage{amsthm}
\usepackage{cancel}
\usepackage{enumitem}
\usepackage{caption}
\setlist[enumerate,1]{leftmargin=2em, itemsep=1em, parsep=0pt}
\setlist[itemize,1]{leftmargin=2em}

\makeatletter
\def\thm@space@setup{%
  \thm@preskip=0.8em plus 0.2em minus 0.2em
  \thm@postskip=0.8em plus 0.2em minus 0.2em
}
\makeatother

\theoremstyle{plain}
\newtheorem{theorem}{Theorem}[section]
\newtheorem{proposition}[theorem]{Proposition}

\theoremstyle{definition}

\usepackage[capitalize,noabbrev]{cleveref}
\usepackage{thm-restate}
\crefformat{section}{Section~#2#1#3}
\crefformat{subsection}{Section~#2#1#3}
\crefformat{subsubsection}{Section~#2#1#3}
\crefformat{equation}{(#2#1#3)}
\crefrangeformat{equation}{(#3#1#4) to (#5#2#6)}
\crefmultiformat{equation}{(#2#1#3)}{ and (#2#1#3)}{, (#2#1#3)}{ and (#2#1#3)}
\crefname{appendix}{Appendix}{Appendices}
\Crefname{appendix}{Appendix}{Appendices}
\crefformat{appendix}{Appendix~#2#1#3}
\AddToHook{cmd/appendix/before}{%
\crefalias{section}{appendix}%
\crefalias{subsection}{appendix}%
\crefalias{subsubsection}{appendix}%
}
\crefname{table}{Table}{Tables}
\crefname{figure}{Figure}{Figures}
\crefname{lemma}{Lemma}{Lemmas}
\crefname{assumption}{Assumption}{Assumptions}

\definecolor{d62728}{RGB}{214,39,40}
\definecolor{1f77b4}{RGB}{31,119,180}
\definecolor{9467bd}{RGB}{148,103,189}
\usepackage{colortbl}
\definecolor{pinegreen}{rgb}{0.0, 0.47, 0.44}
\definecolor{cornellred}{rgb}{0.7, 0.11, 0.11}
\definecolor{cadmiumgreen}{rgb}{0.0, 0.42, 0.24}
\definecolor{spirodiscoball}{rgb}{0.06, 0.75, 0.99}
\definecolor{blizzardblue}{rgb}{0.73, 0.96, 0.99}
\definecolor{aliceblue}{rgb}{0.91, 0.94, 0.97}
\definecolor{darkblue}{RGB}{232, 224, 240}
\definecolor{Red7}{rgb}{0.941, 0.243, 0.243}
\definecolor{Green7}{RGB}{55, 178, 77}

\usepackage{pifont}
\usepackage{array}
\usepackage{xurl}

\newcolumntype{C}[1]{>{\centering\arraybackslash}p{#1}}
\newcolumntype{L}[1]{>{\raggedright\arraybackslash}p{#1}}

\newcommand{\ours}{Lafa}

\title{Large Flow Language Models}

\author[*,1,2]{Arnav Chavan}
\author[*,3]{Chanhyuk Lee}
\author[*,2]{Udbhav Bamba} 
\author[3]{Jaehoon Yoo}
\author[3]{Seunghoon Hong}
\seniorauthor[\dag,4]{Aditi Raghunathan} 
\seniorauthor[\dag,3]{Jinwoo Kim}
\affiliation[1]{UC Berkeley}
\affiliation[2]{Amazon}
\affiliation[3]{KAIST}
\affiliation[4]{Carnegie Mellon University}
\contribution[*]{Equal contribution}
\contribution[\dag]{Equal advising}

\begin{document}

\begin{abstract}
Continuous flow language models offer a path to parallel generation with tools for control and acceleration unavailable to discrete diffusion, yet effective techniques for building them at the billion-parameter scale remain less established.
We present \textbf{\emph{\ours{}} (\emph{La}rge \emph{f}low l\emph{a}nguage model)}, a fast, high-performing flow language model scaled to 8B parameters in a data-efficient manner.
\ours{} is a blockwise flow model with causal attention on the prefix and bidirectional, asynchronous flow within each block.
This design enables efficient adaptation of pretrained autoregressive (AR) models and fast inference via prefix caching and blockwise parallel decoding.
We train \ours{} at 0.5B, 7B, and 8B scales with fewer than 30B tokens, more efficient than previous diffusion models, and evaluate it on instruction following, math, and coding, where the learned flow matches or surpasses the quality of AR and can be distilled into flow maps for one- and few-step generation.
Notably, flow training preserves or even improves pretrained AR capabilities via a multi-token prediction effect, which we leverage in a self-speculative decoding scheme that verifies flow-based drafts autoregressively within a single model.
This yields 2.54$\times$ throughput of an AR model at comparable quality, outperforming multi-token prediction and discrete diffusion methods.
Our findings show that continuous flows can serve as a scalable foundation for language modeling.

\vspace{0.5em}\textbf{\sffamily\bfseries Code:} \href{https://github.com/lafa-gh/large-flow-lms}{\texttt{github.com/lafa-gh/large-flow-lms}}

\vspace{0.1em}\textbf{\sffamily\bfseries Models:} \href{https://huggingface.co/lafa-hf/models}{\texttt{huggingface.co/lafa-hf/models}}

\end{abstract}

\maketitle

\section{Introduction}
\label{sec:intro}

\begin{wrapfigure}[19]{r}{0.42\textwidth}
\centering
\includegraphics[width=\linewidth]{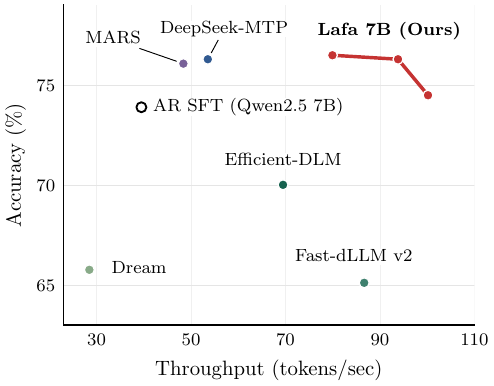}
\caption{%
\textbf{Large flow language models.}
\ours{} 7B attains 2.02--2.54$\times$ the throughput of an AR model at a matched or better quality, outperforming existing MTP and diffusion-based approaches.
}
\label{fig:pareto}
\end{wrapfigure}
Today's large language models are predominantly trained as autoregressive (AR) models that predict the next token from the preceding text~\citep{achiam2023gpt,team2023gemini,guo2025deepseek}.
The parallelism of the transformer architecture and teacher forcing~\citep{vaswani2017attention,williams1989learning} allows these models to be trained efficiently at the scale of billions of parameters and trillions of data tokens.
However, their generation is inherently sequential, producing one token per forward pass, which makes inference latency a bottleneck.
To accelerate inference, prior work has explored approaches to parallel generation, namely multi-token prediction (MTP)~\citep{gloeckle2024better,cai2024medusa,deepseekv3}, discrete diffusion models~\citep{khanna2025mercury,song2025seed,team2026diffusiongemma}, and continuous flow language models~\citep{dieleman2022cdcd,lee2026flm,hu2026elf}.

Distinguished from the other approaches, flow language models perform parallel generation by integrating a continuous flow from noise to token representations~\citep{dieleman2022cdcd}.
By operating in a continuous space, they can leverage powerful techniques developed for controllable and accelerated generation of images and video, many of which are unavailable to discrete methods~\citep{lee2026flm,hu2026elf}.
Yet, existing flow language models have been demonstrated at small scales and on simple tasks~\citep{lee2026flm,chen2026langflow,hu2026elf,batzolis2026cobit,yang2026continuous}, and a prior attempt of scaling has identified challenges in training inefficiency~\citep{gulrajani2024plaid}.
Together, effective techniques for building performant flow language models at the billion-parameter scale remain less established.

In this paper, we present \textbf{\emph{\ours{}} (\emph{La}rge \emph{f}low l\emph{a}nguage model)}, a fast and high-performing flow language model scaled to 8B parameters in a data-efficient manner.
Motivated by the left-to-right directionality of text~\citep{papadopoulos2024arrows}, we design \ours{} as a blockwise flow model, with causal attention on the prefix, and bidirectional but left-to-right-biased asynchronous flow within each block.
This brings its architecture and inductive bias close to AR models, enabling efficient training by adapting pretrained models.
Further, this allows fast inference via prefix caching and parallel flow decoding, which can be further distilled into flow maps for one- and few-step generation.
We train \ours{} with blockwise denoising for the flow, along with AR prediction on the prefix, and find that flow training preserves or even enhances pretrained AR capabilities through a multi-token prediction effect.
We leverage this in a self-speculative decoding scheme that verifies flow-based drafts autoregressively.
With these techniques, we train \ours{} across 0.5B, 7B, and 8B scales with fewer than 30B tokens, cheaper than comparable adaptation of discrete diffusion models that use up to 500B tokens~\citep{dream2024,fu2026edlm}.
Across benchmarks spanning instruction following, reasoning, and coding, \ours{} attains a quality--speed frontier that surpasses similarly sized AR, MTP, and discrete diffusion language models, as shown in \Cref{fig:pareto}.

\begin{figure}[t!]
\centering
\includegraphics[width=\linewidth]{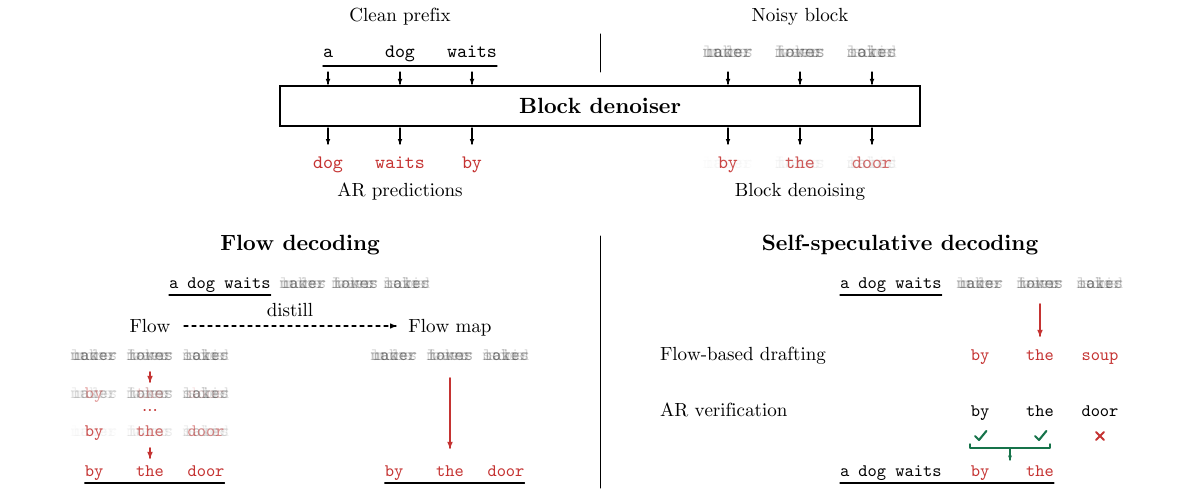}
\caption{\textbf{Overview.}
(Top) \ours{} learns a block denoiser that also performs AR prediction on the prefix.
(Bottom) \ours{} supports flow decoding, which can be further distilled into a flow map for one-step generation.
Together, AR and flow support self-speculative decoding, which accepts flow-based drafts from left to right via AR verification.
}
\label{fig:overview}
\end{figure}

An overview of the approach is in \Cref{fig:overview}.
Our main contributions are as follows:
\begin{itemize}[leftmargin=*]
\item We introduce \ours{}, a large flow language model that learns blockwise asynchronous flow, enabling efficient training by adapting pretrained AR models and fast inference through prefix caching and parallel flows.
\item We identify key design choices for scalability, including the architecture and asynchronous flow formulation, training and flow map distillation objectives, and inference schemes including self-speculative decoding.
\item We scale \ours{} to 7--8B parameters with $<$30B tokens and show that its learned flow matches or surpasses the quality of AR and is distillable into flow maps, and its learning improves pretrained AR capabilities.
\item By jointly leveraging the flow and AR capabilities for self-speculative decoding, \ours{} attains $2.54\times$ the throughput of a strong AR baseline at matched quality, outperforming existing parallel generation methods.
\end{itemize}
Our findings show that continuous flows offer a practical, scalable foundation for parallel language modeling.

\section{Large flow language models}
\label{sec:method}

In this section, we introduce \ours{}, a class of scalable and high-performing large flow language models.
We first formulate blockwise asynchronous flows, then describe the architecture, training, and inference procedures.
See \Cref{app:proofs} for theoretical results and \Cref{app:implementation_details} for algorithmic details.

\subsection{Block asynchronous flows}
\label{sec:async_block_flows}

\paragraph{Blockwise formulation.}
We develop our method upon standard block-autoregressive factorization~\citep{wang2018semi}, as in block discrete diffusion~\citep{arriola2025bd3lm}, retaining causal generation and prefix caching across blocks.
For vocabulary $V$, prefix ${\bf c}\in V^N$, and continuation ${\bf y}\in V^M$, we factorize into blocks of size $B$ (assuming $B\mid M$):
\begin{equation}\label{eq:block_ar}
    p({\bf y}\mid{\bf c})
    = \prod_{r=0}^{M/B-1}
    p\bigl({\bf y}^{rB+1:(r+1)B}
    \mid {\bf c},{\bf y}^{\leq rB}\bigr).
\end{equation}
\paragraph{Embedded flows.}
We embed each block ${\bf y}\in V^B$ using vocabulary embeddings ${\bf E}\in\mathbb{R}^{|V|\times d}$ as ${\bf x}_1\coloneqq({\bf E}_{{\bf y}^1},\ldots,{\bf E}_{{\bf y}^B})^\top\in\mathbb{R}^{B\times d}$.
This reduces the per-token dimension from $|V|$ of one-hot embeddings~\citep{lee2026flm} to $d$, making training more scalable.
Using flow matching~\citep{lipman2023flow,albergo2025stochastic}, we construct a conditional probability path $p_t(\cdot\mid{\bf c})$, $t\in[0,1]$, from Gaussian noise ${\bf x}_0\sim\mathsf{N}(0,\sigma^2I)$ with $\sigma>0$ to the distribution of embedded data ${\bf x}_1$.

\paragraph{Asynchronous flows.}
Standard flow matching uses the interpolant $I_t\coloneqq(1-t){\bf x}_0+t{\bf x}_1$, where $I_t\sim p_t$, with a shared time $t$ for all tokens.
To introduce a left-to-right bias within each block to match the inherent directionality of text~\citep{papadopoulos2024arrows}, we use tokenwise time~\citep{baade2026latent,qian2026learning,wu2023ar,bansal2026token,lee2026asynchronous}, allowing earlier positions to progress faster.
A tokenwise time is a function $\tau:[0,1]\to [0,1]^B$, differentiable and increasing in each coordinate, mapping a base time $t$ to a tokenwise time $\tau^i(t)$ for each $i= 1, ..., B$, satisfying $\tau(0)={\bf 0}$ and $\tau(1)={\bf 1}$.
Given $\tau$, we can construct a probability path $p_t(\cdot\mid{\bf c})$ in base time $t$ through an interpolant $I_t$ defined positionwise~\citep{qian2026learning}:
\begin{equation}\label{eq:async_interpolant}
    I_t^i \coloneqq (1-\tau^i(t))\,{\bf x}_0^i + \tau^i(t)\,{\bf x}_1^i, \qquad i \in[B],
    \qquad I_t\sim p_t(\cdot\mid{\bf c}), \qquad t\in[0,1].
\end{equation}
For each given ${\bf c}$, the probability path $p_t(\cdot\mid{\bf c})$ admits a velocity field $b_t(\cdot; {\bf c}):\mathbb{R}^{B\times d}\to \mathbb{R}^{B\times d}$ that transports a block of noise to clean data embeddings through the following ODE~\citep{qian2026learning}, which can be used for generation:
\begin{equation}\label{eq:flow_velocity}
    \dot{\bf x}_t=b_t({\bf x}_t;{\bf c}),
    \qquad {\bf x}_0\sim p_0,
    \qquad {\bf x}_t\sim p_t(\cdot\mid{\bf c}),
    \qquad t\in[0,1].
\end{equation}
If $\tau$ progresses faster for the earlier tokens (i.e., $\tau^i(t)>\tau^j(t)$ for $i<j$), both interpolation and generation will reveal earlier tokens faster.
In practice, we learn $\tau$ from data instead of fixing it, with details in \Cref{sec:training}.

\paragraph{Block denoiser.}
Instead of learning the asynchronous velocity $b_t(\cdot;{\bf c})$ directly, following related work~\citep{dieleman2022cdcd,lee2026flm}, we learn a block denoiser $D_t(\cdot;{\bf c}):\mathbb{R}^{B\times d}\to (\Delta^{|V|-1})^B$ that predicts clean tokens via tokenwise classification:
\begin{equation}\label{eq:denoising_classifier}
    D_t({\bf x};{\bf c})^i \coloneqq p({\bf y}^i=\cdot\mid{\bf c},I_t={\bf x}), \qquad i \in [B],
\end{equation}
where $\Delta^{|V|-1}$ is the space of probability vectors of length $|V|$.
Notably, learning the denoiser recovers the velocity through the following identity that involves embedding interpolation, which we prove in \Cref{app:proof_async_velocity}:
\begin{equation}\label{eq:async_velocity}
    b_t({\bf x};{\bf c})^i=\frac{\dot\tau^i(t)}{1-\tau^i(t)}(D_t({\bf x};{\bf c})^i{\bf E}-{\bf x}^i).
\end{equation}
The formula scales the direction from noisy input ${\bf x}$ to its denoising $D_t({\bf x};{\bf c}){\bf E}$ by the tokenwise rate of progress $\dot\tau(t)$, thereby implementing the asynchronous progress prescribed by tokenwise time $\tau$.

\subsection{Learning the flow from an autoregressive model}
\label{sec:ar_to_denoiser}
\label{sec:training}

As our formulation keeps a causal structure, we learn the blockwise flow by adapting a pretrained AR model $f$ into a block denoiser.
The challenge is parallel denoising of continuous, noisy embeddings while retaining AR prediction and prefix caching.
Our adaptation combines tokenwise time conditioning, block-bidirectional attention, and constrained updates to pretrained embeddings with joint AR and denoising training.

\paragraph{Tokenwise time conditioning.}
We feed $f$ the embedded prefix ${\bf c}$ followed by the noisy block ${\bf x}$, and condition each position on its corruption level through adaptive layer normalization (adaLN)~\citep{peebles2023scalable}.
Prefix positions always receive time $1$, indicating the absence of noise, while each block position $i\in[B]$ receives its tokenwise time $\tau^i(t)$.
This lets a single forward pass handle both the clean prefix and asynchronous denoising within the block.
We initialize adaLN to the identity, so that adding time conditioning initially leaves the pretrained computation unchanged.
We denote the resulting model by $f_t({\bf x};{\bf c})$.

\begin{wrapfigure}[11]{r}{0.26\textwidth}
\vspace{-1.5em}
\centering
\includegraphics[width=0.85\linewidth]{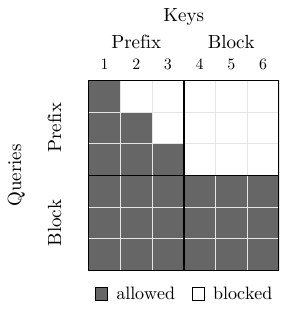}
\caption{Attention mask for block-bidirectional attention.}
\label{fig:attention_mask}
\end{wrapfigure}

\paragraph{Block-bidirectional attention.}
We make attention bidirectional within the noisy block, allowing partially denoised later tokens to inform earlier ones (\Cref{fig:attention_mask}).
Block positions also attend to the entire clean prefix, whose attention remains causal and excludes the noisy block~\citep{tian2025next,jin2026mars}.
Together with conditioning prefix positions on time $1$, this makes their keys and values independent of the block state $({\bf x},t)$, allowing reuse across denoising steps and subsequent blocks.
The prefix outputs retain their AR role: $f_t({\bf x};{\bf c})^i$ predicts ${\bf c}^{i+1}$ for $i<N$, and the first block token ${\bf y}^1$ for $i=N$.
We write these block-independent predictions as $f({\bf c})^i$ and use them to verify the denoiser's drafts in self-speculative decoding (\Cref{sec:inference_schemes}).
This requires retaining or improving pretrained AR capabilities, which we achieve through training.
We find that bidirectional attention within blocks outperforms causal (\Cref{tab:design-ablations}).

\paragraph{Shared and learned embeddings.}
We initialize the vocabulary embeddings ${\bf E}$ from the pretrained model, share them between the prefix and the embedded flow, and train them jointly with the denoiser.
Unconstrained updates, however, can inflate embedding norms relative to the fixed noise scale $\sigma$, making denoising easier by reducing the effective corruption~\citep{dieleman2022cdcd,lee2026flm}.
We therefore fix the L2 norm of each embedding ${\bf E}_v$ to its pretrained value and learn only its direction, allowing adaptation to the flow while preserving the embedding scale.
This outperforms both frozen and unconstrained embeddings (\Cref{tab:design-ablations}).

\paragraph{Training loss.}
We train the block denoiser $\hat{D}_t({\bf x};{\bf c})^i\coloneqq f_t({\bf x};{\bf c})^{N+i}$ jointly with AR on the prefix.
For each training pair $({\bf c},{\bf y})$, we draw $t\sim\mathsf{U}[0,1]$ and noise ${\bf x}_0\sim p_0$, and build $I_t$ via \Cref{eq:async_interpolant} with ${\bf x}_1^i={\bf E}_{{\bf y}^i}$.
We minimize
\begin{equation}\label{eq:loss}
    \mathcal{L}(\hat{D})
    \coloneqq \mathbb{E}_{{\bf c},{\bf y},t,{\bf x}_0}\Bigg[
    \underbrace{\sum_{i=1}^N-\log f({\bf c})^i_{{\bf c}^{i+1}}}_{\text{AR cross entropy}}
    +\underbrace{\sum_{i=1}^B-\log \hat{D}_t(I_t;{\bf c})^i_{{\bf y}^i}}_{\text{denoising cross entropy}}
    +\underbrace{\sum_{i=1}^B
    \bigl|\hat{D}_t(I_t;{\bf c})^i{\bf E}-{\bf E}_{{\bf y}^i}\bigr|^2}_{\text{embedding squared error}}
    \Bigg],
\end{equation}
where ${\bf c}^{N+1}\coloneqq{\bf y}^1$, and the AR and denoising predictions share the transformer parameters.
The denoising cross entropy supervises token probabilities \Cref{eq:denoising_classifier}, while the squared error supervises the mean embedding that determines the flow velocity \Cref{eq:async_velocity}.
Combining these denoising losses improves flow accuracy, and we unexpectedly find that AR supervision improves it further (\Cref{tab:design-ablations}).
Notably, joint training also improves AR accuracy over matched AR SFT (\Cref{sec:exp_ar}), consistent with the benefits of multi-token prediction~\citep{gloeckle2024better,deepseekv3,mahajan2026beyond}.

\begin{figure}[t!]
\centering
\includegraphics[width=\linewidth]{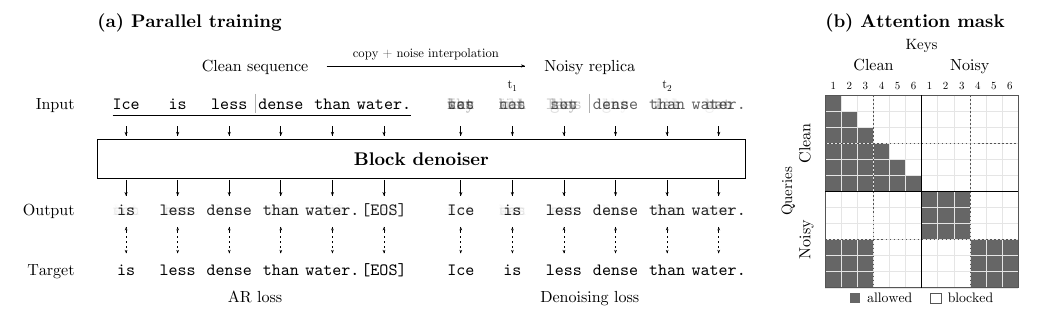}
\caption{Parallelized training over a sequence composed of two blocks of size three.
The uncorrupted sequence uses causal attention and the AR loss, while each noisy block attends to the preceding clean tokens and bidirectionally within itself, and uses the denoising losses.
}
\label{fig:training}
\end{figure}

\paragraph{Parallel training.}
To avoid recomputing overlapping prefixes, we train all blocks of a sequence in parallel.
Following prior work on block diffusion~\citep{arriola2025bd3lm,tian2025next,jin2026mars,liu2025wedlm}, we pair each clean training sequence with a copy partitioned into noisy blocks of size $B$ (\Cref{fig:training}).
Clean positions attend causally, while each noisy block attends only to the preceding clean tokens and bidirectionally within itself.
A single forward pass then produces AR predictions for the clean sequence and denoising predictions for every block (details in \Cref{app:training}).

\begin{wrapfigure}[10]{r}{0.32\textwidth}
\centering
\vspace{-1em}
\includegraphics[width=0.85\linewidth]{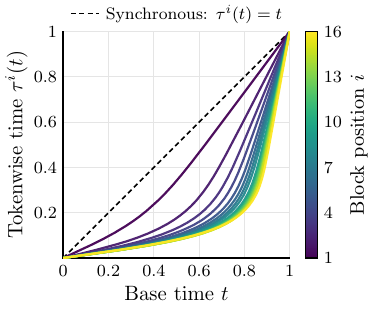}
\caption{Tokenwise time for each block position learned by \ours{} 7B.}
\label{fig:learned_time}
\end{wrapfigure}
\paragraph{Learning tokenwise time.}
We learn the tokenwise time $\tau:[0,1]\to[0,1]^B$ from the denoising loss during training.
Inspired by the loss-aware time warping of CDCD~\citep{dieleman2022cdcd}, we construct each $\tau^i$ as the inverse CDF of the normalized denoising loss across noise levels, estimated online for position $i$ (details in \Cref{app:time_reparam}).
This slows progress where the loss is high, so uniform sampling and discretization of base time $t$ allocate more training samples and integration steps to difficult noise levels.
The learned schedules generally satisfy $\tau^i(t)>\tau^j(t)$ for $i<j$ (\Cref{fig:learned_time}), with earlier positions denoised ahead of later ones, even though this left-to-right bias is not manually imposed.

\paragraph{Self-conditioning.}
Following prior work~\citep{hu2026elf,chen2026langflow,yoo2026self}, we use self-conditioning~\citep{chen2022analog} to improve sample quality.
With a fixed probability during training, we obtain a denoising prediction without gradients, feed it back as an additional input, and compute the loss on a second forward pass (details in \Cref{app:self_conditioning}).
During flow decoding, each Euler step reuses the previous step's prediction for conditioning, requiring no added forward.

\paragraph{Flow map distillation.}
Flow decoding requires multiple denoiser evaluations per block.
To reduce this cost, we distill the learned flow into a flow map $X_{s,t}$ that transports a state directly from time $s$ to $t$~\citep{boffi2026build}.
For this, we make a novel extension of simplex-valued two-time denoisers~\citep{lee2026flm} to asynchronous flows over vocabulary embeddings.
Specifically, $\delta_{s,t}({\bf x};{\bf c})\in(\Delta^{|V|-1})^B$ predicts token probabilities, which we map through ${\bf E}$ before interpolating with the noisy state ${\bf x}$:
\begin{equation}\label{eq:two_time_denoiser}
    X_{s,t}({\bf x};{\bf c})^i
    \coloneqq
    \frac{1-\tau^i(t)}{1-\tau^i(s)}{\bf x}^i
    +\frac{\tau^i(t)-\tau^i(s)}{1-\tau^i(s)}
    \,\delta_{s,t}({\bf x};{\bf c})^i{\bf E},
    \qquad i\in[B],\quad 0\leq s<t\leq1.
\end{equation}
This retains categorical predictions for distillation while transporting states in embedding space.
We initialize the student $\hat{\delta}_{s,t}$ from the frozen teacher $\hat{D}_s$, adding an identity-initialized adaLN for the destination time $t$.
Following semigroup-based distillation~\citep{lee2026flm,yoo2026self}, we match the student's diagonal predictions at $s=t$ to the teacher $\hat{D}_s$ and its off-diagonal predictions at $s\neq t$ to a stop-gradient target formed by composing two shorter student transitions.
We use KL losses for both terms and additionally distill the teacher's prefix AR predictions.
The student preserves prefix caching and supports generation in as few as one step through $\hat{X}_{0,1}$.
The derivation and full objective are in \Cref{app:embedded_flow_maps}.

\subsection{Inference schemes}
\label{sec:inference_schemes}
We generate each block by integrating the learned flow or evaluating its distilled flow map.
We further accelerate generation through self-speculative decoding, using the retained AR predictor $f$ to verify flow-based drafts.
All schemes reuse the prefix key-value (KV) cache and extend it with newly committed tokens.

\paragraph{Flow and flow map decoding.}
Starting from $\hat{\bf x}_0\sim p_0$, we take $K$ steps on the grid of base times $t_k=k/K$, which induces an asynchronous temporal grid $\tau(t_k)\in [0,1]^B$.
With ${\bf p}_k\coloneqq\hat{D}_{t_k}(\hat{\bf x}_{t_k};{\bf c})$, an Euler step gives
\begin{equation}\label{eq:block_flow_ode}
    \hat{\bf x}_{t_{k+1}}^i
    =\hat{\bf x}_{t_k}^i
    +\frac{\tau^i(t_{k+1})-\tau^i(t_k)}{1-\tau^i(t_k)}
    ({\bf p}_k^i{\bf E}-\hat{\bf x}_{t_k}^i),
    \qquad
    \hat{\bf y}^i=\mathsf{argmax}({\bf p}_{K-1}^i),
    \qquad i\in[B],
\end{equation}
where the final prediction determines the generated tokens $\hat{\bf y}$.
Each step reuses the previous prediction for self-conditioning.
We further strengthen prefix conditioning with prefix guidance~\citep{song2025history}, analogous to classifier-free guidance~\citep{ho2022classifier} (details in \Cref{app:prefix_guidance}).
This adds one prefix-free evaluation per step, which is efficient as it processes only the $B$ block tokens.
The distilled flow map model replaces each Euler step with a learned jump, which amounts to using $\hat{\delta}_{t_k,t_{k+1}}$ in place of $\hat{D}_{t_k}$ in \Cref{eq:block_flow_ode}:
\begin{equation}\label{eq:map_inference}
    \hat{\bf x}_{t_{k+1}}
    =\hat{X}_{t_k,t_{k+1}}(\hat{\bf x}_{t_k};{\bf c}),
    \qquad
    \hat{\bf y}^i
    =\mathsf{argmax}\bigl(\hat{\delta}_{t_{K-1},1}(\hat{\bf x}_{t_{K-1}};{\bf c})^i\bigr),
    \qquad i\in[B].
\end{equation}
Setting $K=1$ gives one-step block generation through $\hat{X}_{0,1}$.

\paragraph{Self-speculative decoding.}
We draft a block with one Euler step in tokenwise time \Cref{eq:block_flow_ode}, without prefix guidance.
Since $\tau^i(0)=0$ and $\tau^i(1)=1$, this reduces to a single denoiser evaluation from $\hat{\bf x}_0\sim p_0$:
\begin{equation}\label{eq:denoiser_draft}
    {\bf p}^i\coloneqq\hat{D}_0(\hat{\bf x}_0;{\bf c})^i,
    \qquad
    \hat{\bf x}_1^i={\bf p}^i{\bf E},
    \qquad
    \hat{\bf y}^i=\mathsf{argmax}({\bf p}^i),
    \qquad i\in[B].
\end{equation}
One-step drafts can be inconsistent across positions, so we verify them with the retained AR predictor $f$~\citep{stern2018blockwise}.
We commit the longest draft prefix agreeing with greedy AR predictions, and if the first token disagrees, we instead commit $\mathsf{argmax}(f({\bf c})^N)$.
The resulting output therefore matches greedy AR decoding with $f$.
To improve efficiency, following fused drafting and verification schemes~\citep{stern2018blockwise}, we verify the previous draft while generating candidate next drafts with one-step flow decoding in a single forward pass.
Here, each candidate conditions on a possible accepted prefix, and after verification, we keep the candidate corresponding to the verified acceptance length (details in \Cref{app:denoiser_inference}).
As an alternative to self-speculation, a draft may be accepted by confidence thresholding~\citep{jin2026mars}, which we detail in \Cref{app:confidence_results}.

\section{Experiments}
\label{sec:experiments}

In our experiments, we evaluate whether \ours{} learns flows that reach AR quality through data-efficient adaptation, and whether flow map distillation improves few-step generation (\Cref{sec:exp_flow}).
We then examine how flow training affects AR accuracy (\Cref{sec:exp_ar}) and use the retained AR predictor for self-speculative decoding, measuring both parallelism and wall-clock throughput (\Cref{sec:exp_selfspec}).
Finally, we ablate the main training and inference choices (\Cref{sec:exp_ablations}).
Evaluation details are in \Cref{app:eval_details}.


\paragraph{Models and data.}
We build \ours{} at 0.5B and 7B by adapting the corresponding Qwen2.5 Instruct checkpoints~\citep{qwen2025qwen25technicalreport} as described in \Cref{sec:method}.
At 8B, we use Qwen3 8B Base~\citep{yang2025qwen3technicalreport}, as further training of its heavily post-trained Instruct counterpart degraded performance in our experiments.
We train all models on Dolci-Instruct-SFT~\citep{olmo2025olmo} for 5--15 epochs (9--27B tokens), using at least $18\times$ fewer training tokens than the reported budgets of Dream, LLaDA, and Efficient-DLM at comparable scales~\citep{dream2024,llada2025,fu2026edlm}.
We use 0.5B for ablations, 7B for the main results, and 8B for scale-matched comparisons with additional baselines.

\paragraph{Baselines.}
We compare against multi-token prediction, block-parallel decoding, and discrete diffusion methods that represent existing approaches for parallel language modeling (\Cref{sec:related}).
At 0.5B, we evaluate MTP~\citep{gloeckle2024better}, DeepSeek-MTP~\citep{deepseekv3}, MARS~\citep{jin2026mars}, SBD~\citep{gat2025set}, BD3-LM~\citep{arriola2025bd3lm}, and PTP~\citep{draxler2026parallel}.
At 7--8B, we additionally compare against Fast-dLLM~v2~\citep{wu2026fast}, Dream~\citep{dream2024}, LLaDA~\citep{llada2025}, and Efficient-DLM~\citep{fu2026edlm}.
Baselines marked $^\dagger$ are trained by us with the same AR initialization, data, and training budget as \ours{}, providing controlled comparisons.
The remaining baselines use released checkpoints with different training data and budgets.
We also fine-tune each AR initialization on the same data to obtain AR SFT, our reference for both quality and speed.
Unless stated otherwise, we use each method's native decoding protocol.

\paragraph{Benchmarks.}
We evaluate on six benchmarks spanning instruction following, reasoning, and code generation: IFEval (0-shot)~\citep{zhou2023instruction}, BBH (3-shot)~\citep{suzgun2023challenging}, MMLU-Pro (0-shot)~\citep{wang2024mmlu}, GPQA (0-shot)~\citep{rein2023gpqa}, GSM8K (0-shot)~\citep{cobbe2021training}, and HumanEval (0-shot)~\citep{chen2021evaluating}.
%
%
These downstream tasks assess generation quality without relying on generative perplexity, whose limitations are documented by \citet{franca2026hacking}.
All methods are evaluated with lm-evaluation-harness~\citep{eval-harness} using each benchmark's default prompts.
We report prompt-level strict accuracy for IFEval, pass@1 for HumanEval, and accuracy for the remaining tasks.
Generation is capped at 1024 tokens per sample, except 512 for GSM8K and 640 for HumanEval.
We report six-benchmark macro-averages for single-token accuracy and self-speculative accuracy and TPF (\Cref{tab:ntp_main,fig:selfspec_frontier,tab:selfspec_lafa_sixbench}).
For flow and flow map decoding, throughput, and ablations, we average over IFEval, GSM8K, and HumanEval, covering instruction following, math, and coding, considering their more costly evaluations.


\subsection{Flow and flow map decoding}
\label{sec:exp_flow}

\begin{figure}[t!]
\centering
\vspace{-1.5em}
\begin{minipage}[c]{0.48\textwidth}
\centering
\includegraphics[width=\linewidth]{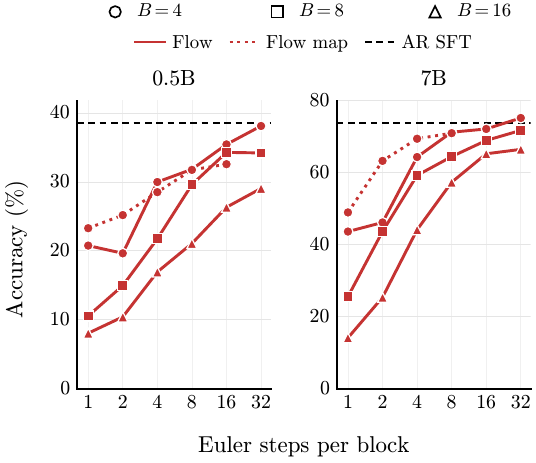}
\caption{Flow and flow map decoding results.
Accuracy averaged over IFEval, GSM8K, and HumanEval.}
\label{fig:multitoken_frontier}
\end{minipage}\hfill
\begin{minipage}[c]{0.48\textwidth}
\centering
\vspace{-2em}
\footnotesize
\captionof{table}{Single-token decoding accuracy (\%)~$\uparrow$ averaged across all six benchmarks.
$\dagger$ denotes matched initialization, data, and training.}
\label{tab:ntp_main}
\begin{tabular}{lccc}
\toprule
{Method} & {0.5B} & {7B} & {8B} \\
\midrule
AR SFT$^\dagger$ & 28.7 & 58.4 & 56.8 \\
MTP$^\dagger$~\citep{gloeckle2024better} & 25.9 & 53.9 & -- \\
DeepSeek-MTP$^\dagger$~\citep{deepseekv3} & 30.2 & \underline{60.0} & -- \\
MARS$^\dagger$~\citep{jin2026mars} & \underline{30.4} & 59.7 & -- \\
SBD$^\dagger$~\citep{gat2025set} & 27.6 & -- & -- \\
BD3-LM$^\dagger$~\citep{arriola2025bd3lm} & 22.8 & -- & -- \\
Fast-dLLM~v2$^\dagger$~\citep{wu2026fast} & -- & 49.3 & -- \\
Fast-dLLM~v2~\citep{wu2026fast} & -- & 54.4 & -- \\
Dream~\citep{dream2024} & -- & 55.6 & -- \\
LLaDA~\citep{llada2025} & -- & -- & 43.5 \\
Efficient-DLM~\citep{fu2026edlm} & -- & -- & \textbf{63.5} \\
\rowcolor{blue!5} \textbf{\ours{} (Ours)$^\dagger$}  & \textbf{31.6} & \textbf{60.9} & \underline{63.1} \\
\bottomrule
\end{tabular}
\end{minipage}
\end{figure}



We first evaluate flow generation with block sizes $B\in\{4,8,16\}$, varying the number of Euler integration steps in \Cref{eq:block_flow_ode}.
We also test whether flow map distillation improves few-step generation.
\Cref{fig:multitoken_frontier} summarizes the results (per-benchmark scores in \Cref{app:full_results}).
Across scales and block sizes, flow accuracy generally improves with more integration steps, demonstrating effective iterative refinement.
With sufficient steps, flow decoding with block size four approaches AR SFT accuracy at 0.5B and surpasses it at 7B.
At a fixed number of steps, larger blocks generally trade accuracy for greater parallelism.
For flow map decoding, we distill the models with block size four and observe improved accuracy in the one- and two-step regimes.
The gains are largest at 7B, where the two-step student approaches the four-step teacher; the student also outperforms the teacher at four steps and reaches similar accuracy at eight steps.
Nevertheless, a quality gap remains relative to flow decoding with more integration steps.
We leave better incorporation of self-conditioning and prefix guidance into flow map distillation to future work.

\subsection{Autoregressive decoding}
\label{sec:exp_ar}

We next examine how joint flow training affects the retained AR capability (\Cref{sec:training}), evaluating \ours{} and baselines under single-token decoding.
As shown in \Cref{tab:ntp_main}, \ours{} improves average accuracy over matched AR SFT across all three model scales (per-benchmark results in \Cref{app:full_results}).
It also outperforms the matched multi-token and block-parallel baselines on average at 0.5B and 7B.
At 8B, \ours{} achieves similar accuracy to Efficient-DLM evaluated with single-token unmasking, while using at least $18\times$ fewer training tokens.
These results show that learning the flow can also improve AR prediction, consistent with the benefits of multi-token training~\citep{gloeckle2024better,deepseekv3,mahajan2026beyond}.
We therefore use this retained AR predictor to verify flow-based drafts for self-speculative decoding (\Cref{sec:exp_selfspec}).

\subsection{Self-speculative decoding}
\label{sec:exp_selfspec}

\begin{figure}[t!]
\centering
\includegraphics[width=\linewidth]{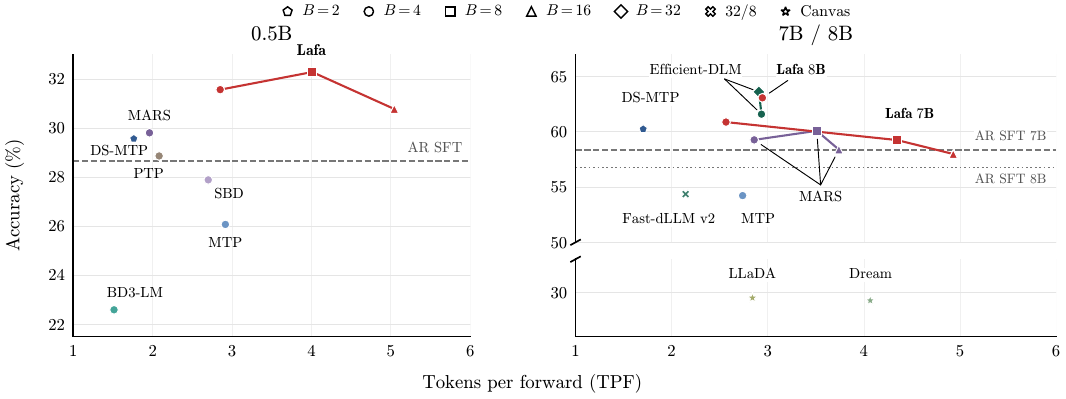}
\caption{\textbf{Self-speculative decoding.} Accuracy $\uparrow$ vs. TPF $\uparrow$ averaged over all six benchmarks.
In the right panel, Efficient-DLM and LLaDA are 8B and the other baselines are 7B.
``Canvas'' denotes iterative completion over the full sequence, while ``32/8'' denotes decoding 32-token blocks in four 8-token sub-blocks.
Further details in \Cref{app:full_results}.
}
\label{fig:selfspec_frontier}
\end{figure}

\begin{table}[t!]
\centering
\footnotesize
\caption{Self-speculative decoding accuracy (\%) $\uparrow$ and TPF $\uparrow$.}
\label{tab:selfspec_lafa_sixbench}
\resizebox{\linewidth}{!}{%
\begin{tabular}{lccccccccccccccc}
\toprule
 & & \multicolumn{2}{c}{IFEval} & \multicolumn{2}{c}{BBH} & \multicolumn{2}{c}{MMLU-Pro} & \multicolumn{2}{c}{GPQA} & \multicolumn{2}{c}{GSM8K} & \multicolumn{2}{c}{HumanEval} & \multicolumn{2}{c}{Average} \\
\cmidrule(lr){3-4}\cmidrule(lr){5-6}\cmidrule(lr){7-8}\cmidrule(lr){9-10}\cmidrule(lr){11-12}\cmidrule(lr){13-14}\cmidrule(lr){15-16}
Method & $B$ & Acc. & TPF & Acc. & TPF & Acc. & TPF & Acc. & TPF & Acc. & TPF & Acc. & TPF & Acc. & TPF \\
\midrule
\rowcolor{gray!10} \multicolumn{16}{c}{\emph{0.5B scale}} \\
Qwen2.5 SFT & -- & 48.4 & 1.0 & 26.3 & 1.0 & \underline{11.9} & 1.0 & 17.9 & 1.0 & 32.0 & 1.0 & 35.4 & 1.0 & 28.7 & 1.0 \\
\rowcolor{blue!5} & 4 & \underline{55.1} & 2.2 & \underline{27.6} & 3.3 & \textbf{12.1} & 3.1 & \underline{21.9} & 2.9 & 36.1 & 2.7 & 36.6 & 2.9 & \underline{31.6} & 2.9 \\
\rowcolor{blue!5} & 8 & \textbf{55.3} & 2.7 & \textbf{27.9} & 5.1 & 10.0 & 5.5 & \textbf{23.7} & 4.2 & \textbf{37.8} & 3.6 & \underline{39.0} & 3.0 & \textbf{32.3} & 4.0 \\
\rowcolor{blue!5} \multirow{-3}{*}{\textbf{\ours{} (Ours)}} & 16 & 54.2 & 3.1 & 27.0 & 6.3 & 9.1 & 9.2 & 18.1 & 4.8 & \underline{36.7} & 3.6 & \textbf{39.6} & 3.3 & 30.8 & 5.0 \\
\midrule
\rowcolor{gray!10} \multicolumn{16}{c}{\emph{7B scale}} \\
Qwen2.5 SFT & -- & 73.8 & 1.0 & \underline{59.2} & 1.0 & \underline{41.6} & 1.0 & 27.9 & 1.0 & 69.1 & 1.0 & \textbf{78.7} & 1.0 & 58.4 & 1.0 \\
\rowcolor{blue!5} & 4 & \textbf{78.9} & 1.7 & \textbf{63.0} & 3.4 & \textbf{41.8} & 2.4 & \textbf{31.0} & 2.3 & 73.8 & 2.8 & \underline{76.8} & 2.8 & \textbf{60.9} & 2.6 \\
\rowcolor{blue!5} & 8 & \underline{77.5} & 2.8 & 59.1 & 6.3 & 39.4 & 3.7 & \underline{28.1} & 4.0 & \textbf{75.3} & 4.8 & 76.2 & 4.6 & \underline{59.3} & 4.3 \\
\rowcolor{blue!5} \multirow{-3}{*}{\textbf{\ours{} (Ours)}} & 16 & 74.9 & 2.8 & 57.3 & 7.9 & 39.0 & 3.8 & \underline{28.1} & 4.0 & \underline{75.0} & 5.8 & 73.8 & 5.4 & 58.0 & 4.9 \\
\midrule
\rowcolor{gray!10} \multicolumn{16}{c}{\emph{8B scale}} \\
Qwen3 SFT & -- & 38.1 & 1.0 & \textbf{71.6} & 1.0 & \textbf{56.2} & 1.0 & 29.5 & 1.0 & 62.9 & 1.0 & \textbf{82.3} & 1.0 & 56.8 & 1.0 \\
\rowcolor{blue!5} \textbf{\ours{} (Ours)} & 4 & \textbf{62.3} & 2.4 & 70.7 & 3.4 & \textbf{56.2} & 3.2 & \textbf{31.7} & 3.1 & \textbf{77.8} & 2.6 & 79.9 & 2.9 & \textbf{63.1} & 2.9 \\
\bottomrule
\end{tabular}}
\end{table}

We now evaluate self-speculative decoding, which verifies one-step flow drafts with the retained AR predictor (\Cref{sec:inference_schemes}).
We measure both parallelism, in tokens committed per forward pass (TPF), and end-to-end wall-clock throughput.
Each block size uses its corresponding trained checkpoint, and self-speculation preserves that checkpoint's greedy AR output.

\paragraph{Parallelism.}
\Cref{fig:selfspec_frontier,tab:selfspec_lafa_sixbench} report accuracy and TPF averaged over all six benchmarks (protocols detailed in \Cref{app:full_results}).
At 0.5B, \ours{} with block size eight achieves about twice the TPF of DeepSeek-MTP, MARS, and PTP while attaining higher average accuracy.
At 7B, increasing the block size from four to sixteen nearly doubles TPF while keeping accuracy close to AR SFT.
With block size four, \ours{} exceeds DeepSeek-MTP in both accuracy and TPF; with block size eight, it matches the accuracy of MARS with block size four at about $1.5\times$ its TPF.
At 8B, \ours{} achieves similar accuracy and TPF to Efficient-DLM despite using at least $18\times$ fewer training tokens.
Across scales, \ours{} achieves approximately 2.5--5 TPF while remaining close to or above AR SFT accuracy.


\begin{wrapfigure}[19]{r}{0.5\textwidth}
\centering
\includegraphics[width=\linewidth]{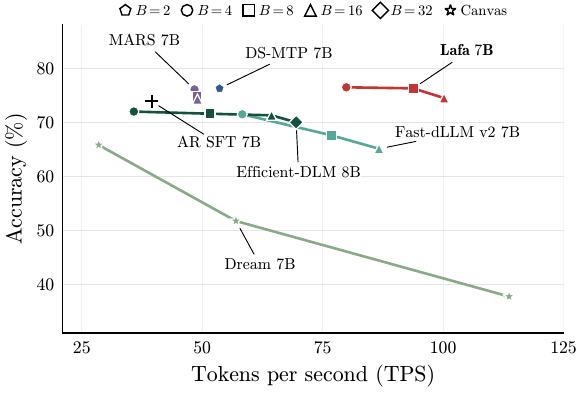}
\caption{Accuracy $\uparrow$ vs. end-to-end throughput measured in tokens per second (TPS) $\uparrow$ at 7--8B scales, averaged over IFEval, GSM8K, and HumanEval.
}
\label{fig:throughput}
\end{wrapfigure}

\paragraph{Wall-clock throughput.}
Since forward-pass costs differ across methods, we also measure end-to-end throughput at 7--8B scales on a single H100 GPU with batch size one, using each method's native decoder.
\Cref{fig:throughput} reports accuracy and throughput in tokens per second (TPS), averaged over IFEval, GSM8K, and HumanEval (settings and per-benchmark results in \Cref{app:throughput}).
With block size sixteen, \ours{} 7B achieves $2.54\times$ the throughput of 7B AR SFT at comparable average accuracy.
With block size four, it is about twice as fast while also more accurate on average.
DeepSeek-MTP and MARS match or exceed AR SFT accuracy but achieve at most $1.4\times$ its throughput.
Diffusion baselines fall below 7B AR SFT accuracy, faster settings showing larger losses.
Thus, while flow decoding requires multiple integration steps to reach AR quality (\Cref{sec:exp_flow}), verifying one-step flow drafts enables substantially faster generation at comparable quality.

Accepting drafts by confidence thresholding instead of AR verification trades accuracy for speed.
We therefore use self-speculation as our default fast decoding mode and report confidence decoding results in \Cref{app:confidence_results}.

\subsection{Ablations}
\label{sec:exp_ablations}

\begin{wraptable}[18]{r}{0.46\textwidth}
\vspace{-4em}
\caption{\textbf{Ablations.}
We use flow decoding and report mean accuracy over IFEval, GSM8K, and HumanEval.
}
\label{tab:design-ablations}
\footnotesize
\centering
\begin{tabular}{p{\dimexpr0.64\linewidth-2\tabcolsep\relax}>{\centering\arraybackslash}p{\dimexpr0.36\linewidth-2\tabcolsep\relax}}
\toprule
Configuration & Average (\%) \\
\midrule
\multicolumn{2}{c}{\emph{Training}} \\
\midrule
Fully causal attention & 32.7 \\
\rowcolor{blue!5} \textbf{Block-bidirectional attention} & \textbf{38.0} \\
\midrule
Synchronous time & 36.2 \\
\rowcolor{blue!5} \textbf{Asynchronous tokenwise time} & \textbf{38.0} \\
\midrule
Frozen embeddings & 34.4 \\
Unconstrained training & 36.9 \\
\rowcolor{blue!5} \textbf{Norm-frozen training} & \textbf{38.0} \\
\midrule
Denoising CE & 33.6 \\
Denoising MSE & 29.7 \\
Denoising CE \& MSE & 35.8 \\
\rowcolor{blue!5} \textbf{Denoising CE \& MSE + AR} & \textbf{38.0} \\
\midrule
\multicolumn{2}{c}{\emph{Inference}} \\
\midrule
No self-conditioning & 32.3 \\
\rowcolor{blue!5} \textbf{Self-conditioning} & \textbf{38.1} \\
\midrule
No guidance & 34.3 \\
\rowcolor{blue!5} \textbf{Prefix guidance} & \textbf{38.1} \\
\bottomrule
\end{tabular}
\end{wraptable}

We ablate the main training and inference choices in \Cref{sec:method} using flow decoding to assess their effect on the learned flow.
All comparisons use 0.5B models with block size four and 32 Euler integration steps (\Cref{tab:design-ablations}).
For training, block-bidirectional attention and learned tokenwise time both improve accuracy over their causal and synchronous counterparts.
Learning embedding directions while fixing their norms outperforms both frozen embeddings and unconstrained updates.
Combining denoising cross entropy and embedding squared error improves accuracy over either objective alone.
Adding prefix AR loss improves flow accuracy further, showing that retaining AR capabilities need not compromise flow generation.
For inference, both self-conditioning and prefix guidance improve accuracy.
Self-conditioning provides the larger gain while requiring no additional model evaluations during decoding.
These results motivate better incorporation of both components into the flow map.

\section{Related work}
\label{sec:related}

In this section, we review existing representative approaches for parallel language modeling.

\paragraph{Block-autoregressive language models.}
Block-autoregressive models reduce sequential generation by producing blocks of tokens conditioned on preceding blocks.
Approaches to parallel block generation include multi-token prediction (MTP) heads or modules~\citep{gloeckle2024better,deepseekv3}, and predictors that take masked or noisy blocks as input~\citep{arriola2025bd3lm,gat2025set,jin2026mars}.
Many parallel decoders predict tokens independently within an update, which can degrade quality when many tokens are committed at once~\citep{lee2026flm}.
Some methods mitigate this through AR verification or confidence-based acceptance of a draft prefix~\citep{stern2018blockwise,deepseekv3,jin2026mars}.
Some recent methods learn deterministic noise-to-token maps that can capture joint dependencies without independent output sampling~\citep{draxler2026parallel,zhang2026gumbel}.

\paragraph{Discrete diffusion language models.}
Discrete diffusion models generate text by reversing a discrete corruption process, such as masking or random token replacement, through iterative denoising~\citep{sahoo2024simple,shi2024simplified,lou2024discrete,schiff2025simple}.
As with other parallel decoders, independent sampling within each update can limit quality in few-step generation~\citep{austin2021d3pm,gat2024discrete}.
These models have been scaled to billions of parameters~\citep{dream2024,fu2026edlm,llada2025,team2026diffusiongemma}, with generation performed over entire sequences or individual blocks.
Block diffusion~\citep{arriola2025bd3lm} combines iterative denoising within each block with autoregressive generation and prefix caching across blocks.

\paragraph{Continuous diffusion and flow language models.}
Continuous diffusion and flow models generate text through noisy continuous representations, including token embeddings, simplex representations, and binary encodings~\citep{li2022diffusionlm,dieleman2022cdcd,han2023ssd,wu2023ar,strudel2022self,chen2022analog,lovelace2023latent,gulrajani2024plaid}.
Continuous flows transport joint sequence representations and admit flow maps that can be learned for few-step generation~\citep{boffi2026build,lee2026flm}.
Recent work has demonstrated competitive generation quality and improved few-step decoding~\citep{lee2026flm,chen2026langflow,hu2026elf,batzolis2026cobit,yang2026continuous,agarwal2026posterior,yoo2026self}, while a few recent efforts have explored billion-parameter scales~\citep{tae2025tess,davis2026scaling,zhuang2026bitlm,yu2026dsl}.
Related approaches also study autoregression and diffusion over autoencoded continuous representations~\citep{shao2025continuous,guo2026continuous,jiang2026textldm,liang2026aurora}.
Realizing the benefits of continuous flows at large and practical scale calls for data-efficient training and fast inference.
\ours{} addresses both by aligning blockwise asynchronous flows with the causal structure of pretrained AR models.
This enables efficient adaptation and prefix caching, while retained AR prediction supports self-speculative generation.

\section{Conclusion}
\label{sec:conclusion}

We introduced \ours{}, a blockwise flow language model scaled to 8B parameters by adapting pretrained AR models with fewer than 30B training tokens.
\ours{} combines causal generation across blocks with learned asynchronous flows within each block, preserving prefix caching while enabling parallel generation.
With sufficient integration steps, flow decoding approaches or surpasses AR SFT accuracy, and flow map distillation improves few-step generation.
Joint training also improves AR accuracy over matched AR SFT.
This retained AR predictor verifies one-step flow drafts within the same model, enabling self-speculative decoding.
With this approach, \ours{} 7B achieves $2.54\times$ the throughput of the 7B AR SFT baseline at comparable average accuracy.
These results demonstrate a practical path to scaling continuous flows for language modeling through data-efficient AR adaptation.
Improving flow map distillation, particularly its incorporation of self-conditioning and prefix guidance, remains a promising direction for further accelerating flow generation.

\subsection*{Acknowledgments}

The authors would like to thank Sheel Shah, Manan Agarwal, Nick Boffi, Sander Dieleman, Jiaxin Shi, and Kyunghyun Cho for valuable discussions.

\bibliographystyle{unsrtnat}
\bibliography{main}

\newpage
\appendix

\section{Theoretical results}
\label{app:proofs}

\subsection{Asynchronous flow velocity and denoiser}
\label{app:proof_async_velocity}

In this subsection, we derive the relationship between the velocity field and the denoiser of the asynchronous flow \Cref{eq:async_velocity} used in \Cref{sec:async_block_flows}.

\begin{proposition}
\label{prop:async_velocity}
For the interpolant \Cref{eq:async_interpolant}, the velocity $b_t({\bf x};{\bf c})=\mathbb{E}[\dot I_t\mid{\bf c},I_t={\bf x}]$ satisfies, at each position $i$,
\begin{equation*}
    b_t({\bf x};{\bf c})^i=\frac{\dot\tau^i(t)}{1-\tau^i(t)}\big(D_t({\bf x};{\bf c})^i{\bf E}-{\bf x}^i\big),
    \qquad t\in[0,1).
\end{equation*}
\end{proposition}

\begin{proof}
The time derivative of the interpolant is 
\begin{equation}
    \dot I_t^i=\dot\tau^i(t)({\bf x}_1^i-{\bf x}_0^i), 
\end{equation}
and its conditional expectation gives the velocity~\citep{albergo2025stochastic,qian2026learning}.
Solving \Cref{eq:async_interpolant} for the noise gives
\begin{equation*}
    {\bf x}_0^i=\frac{I_t^i-\tau^i(t){\bf x}_1^i}{1-\tau^i(t)},
    \qquad\text{so}\qquad
    {\bf x}_1^i-{\bf x}_0^i=\frac{{\bf x}_1^i-I_t^i}{1-\tau^i(t)}.
\end{equation*}

Substituting into $b_t({\bf x};{\bf c})=\mathbb{E}[\dot I_t\mid{\bf c},I_t={\bf x}]$ with $\dot I_t^i=\dot\tau^i(t)({\bf x}_1^i-{\bf x}_0^i)$,
\begin{equation*}
    b_t({\bf x};{\bf c})^i
    =\frac{\dot\tau^i(t)}{1-\tau^i(t)}\,\mathbb{E}\big[{\bf x}_1^i-I_t^i\mid{\bf c},I_t={\bf x}\big]
    =\frac{\dot\tau^i(t)}{1-\tau^i(t)}\Big(\mathbb{E}\big[{\bf x}_1^i\mid{\bf c},I_t={\bf x}\big]-{\bf x}^i\Big).
\end{equation*}
Since ${\bf x}_1^i={\bf E}_{{\bf y}^i}$, the conditional mean is the posterior mean embedding under \Cref{eq:denoising_classifier},
\begin{equation}
    \mathbb{E}[{\bf x}_1^i\mid{\bf c},I_t={\bf x}]=\sum_{v\in V}p({\bf y}^i=v\mid{\bf c},I_t={\bf x})\,{\bf E}_v=D_t({\bf x};{\bf c})^i{\bf E}.
\end{equation}
\end{proof}

\subsection{Asynchronous embedded flow maps}
\label{app:embedded_flow_maps}

In this subsection, we derive the embedded flow map used in \Cref{sec:training}, first characterizing the two-time denoiser \Cref{eq:two_time_denoiser} and then giving the full distillation objective \Cref{eq:psd_loss}.

\paragraph{Setup.}
Throughout, the prefix ${\bf c}$ is fixed, and the vocabulary embeddings ${\bf E}\in\mathbb{R}^{|V|\times d}$ and tokenwise schedules are held fixed.
A block has positions $i=1,...,B$ with clean embeddings ${\bf x}_1^i={\bf E}_{{\bf y}^i}$.
For each position $i$, let $\tau^i$ be continuously differentiable with $\dot\tau^i(t)>0$ for $t<1$, $\tau^i(0)=0$, and $\tau^i(1)=1$, and write $\tau_t^i\coloneqq\tau^i(t)$.
Since all tokenwise times are driven by the single global time $t$, they trace a one-dimensional monotone curve in $[0,1]^B$, and the flow map is indexed by two global times~\citep{lee2026asynchronous, bansal2026token}. Let $\hat D_t$ be the frozen teacher denoiser (\Cref{sec:training}). As in \Cref{eq:async_velocity}, it induces the velocity and flow map
\begin{equation}
    \begin{aligned}
    b_t({\bf x};{\bf c})^i
    &=
    \frac{\dot\tau^i(t)}{1-\tau_t^i}
    \big(\hat D_t({\bf x};{\bf c})^i{\bf E}-{\bf x}^i\big),\\
    \partial_t X_{s,t}({\bf x};{\bf c})
    &=b_t(X_{s,t}({\bf x};{\bf c});{\bf c}),
    \qquad X_{s,s}({\bf x};{\bf c})={\bf x}.
    \end{aligned}
\end{equation}
We work on $0\leq s\leq t<1$ and assume unique ODE solutions and that all denoisers are continuous and differentiable in time.
Expressions at the endpoint $t=1$, used in flow map inference, are understood as limits when they exist.

\paragraph{Two-time denoiser.}

For $0\leq s<t<1$, we define the average velocity $v_{s,t}$ between $s$ and $t$ as \citep{boffi2026build}
\begin{equation}\label{eq:app_average_velocity}
    v_{s,t}({\bf x};{\bf c})\coloneqq\frac{X_{s,t}({\bf x};{\bf c})-{\bf x}}{t-s}.
\end{equation}
To express transport between two times through a denoising prediction, we define a two-time denoiser~\citep{lee2026flm, lu2026one} $\delta_{s,t}({\bf x};{\bf c})\in(\Delta^{|V|-1})^B$ through its mean embedding,
\begin{equation}\label{eq:app_two_time_denoiser}
    \delta_{s,t}({\bf x};{\bf c})^i{\bf E}
    ={\bf x}^i
    +\frac{(1-\tau_s^i)(t-s)}{\tau_t^i-\tau_s^i}
      v_{s,t}({\bf x};{\bf c})^i.
\end{equation}
Substituting \Cref{eq:app_average_velocity} and rearranging gives \Cref{eq:two_time_denoiser}.
Since $|V|\gg d$, distinct probability vectors can share the same mean embedding.
Thus \Cref{eq:app_two_time_denoiser} fixes only the mean embedding $\delta_{s,t}{\bf E}$ and not $\delta$ itself.
The following proposition constructs a specific $\delta$ and identifies the conditions under which a learned denoiser recovers the flow.
\begin{proposition}
\label{prop:embedded_characterization}
There exists a two-time denoiser $\delta$ satisfying \Cref{eq:app_two_time_denoiser} with the following two conditions.
\begin{enumerate}
    \item (Diagonal condition) $\delta_{s,t}\to \hat D_s$ as $t\to s$, so that $\delta_{s,s}\coloneqq \hat D_s$.
    \item (Semigroup condition) For $s<u<t$,
\begin{equation}\label{eq:app_denoiser_semigroup}
\begin{aligned}
    \delta_{s,t}({\bf x};{\bf c})^i=\gamma^i\delta_{s,u}({\bf x};{\bf c})^i
      +(1-\gamma^i)\delta_{u,t}(X_{s,u}({\bf x};{\bf c});{\bf c})^i,
      \qquad
    \gamma^i = \frac{(1-\tau_t^i)(\tau_u^i-\tau_s^i)}
    {(1-\tau_u^i)(\tau_t^i-\tau_s^i)}.
\end{aligned}
\end{equation}
\end{enumerate}
Conversely, for any two-time denoiser $\delta_{s,t}$ satisfying the diagonal and semigroup conditions, its associated flow map defined by \Cref{eq:two_time_denoiser} equals the true flow map $X_{s,t}$.
\end{proposition}

\begin{proof}
By the velocity in the setup,
\begin{equation}\label{eq:embedded_integrating_factor}
    \frac{{\rm d}}{{\rm d}r}\frac{X_{s,r}^i}{1-\tau_r^i}
    =\frac{(1-\tau_r^i)\,\partial_r X_{s,r}^i+\dot\tau^i(r)\,X_{s,r}^i}{(1-\tau_r^i)^2}
    =w^i(r)\,\hat D_r(X_{s,r})^i{\bf E},
    \qquad
    w^i(r)\coloneqq\frac{\dot\tau^i(r)}{(1-\tau_r^i)^2},
\end{equation}
where we suppress ${\bf x}$ and ${\bf c}$.
Define
\begin{equation}\label{eq:exact_two_time_denoiser}
    \delta_{s,t}^i\coloneqq\frac{1}{Z_{s,t}^i}\int_s^t w^i(r)\,\hat D_r(X_{s,r})^i\,{\rm d}r,
    \qquad
    Z_{s,t}^i\coloneqq\int_s^t w^i(r)\,{\rm d}r=\frac{\tau_t^i-\tau_s^i}{(1-\tau_t^i)(1-\tau_s^i)}.
\end{equation}
This is a weighted average of probability vectors, hence a probability vector.

Integrating \Cref{eq:embedded_integrating_factor} over $[s,t]$ and multiplying by $1-\tau_t^i$ gives
$X_{s,t}^i=\frac{1-\tau_t^i}{1-\tau_s^i}{\bf x}^i+\frac{\tau_t^i-\tau_s^i}{1-\tau_s^i}\delta_{s,t}^i{\bf E}$,
which is \Cref{eq:app_two_time_denoiser}.

\textit{Diagonal condition.}
As $t\to s$, the normalized weights $w^i/Z_{s,t}^i$ concentrate at $r=s$, so $\delta_{s,t}\to \hat D_s$ by continuity of the teacher along the trajectory.

\textit{Semigroup condition.}
Split the integral at $u$.
The part over $[s,u]$ is $Z_{s,u}^i\delta_{s,u}^i$.
For $r\geq u$, $X_{s,r}=X_{u,r}(X_{s,u})$, so the part over $[u,t]$ is $Z_{u,t}^i\delta_{u,t}(X_{s,u})^i$.
Dividing by $Z_{s,t}^i=Z_{s,u}^i+Z_{u,t}^i$ gives \Cref{eq:app_denoiser_semigroup} with $\gamma^i=Z_{s,u}^i/Z_{s,t}^i$.

\textit{Converse.}
Let $\hat X$ be defined from $\delta$ by \Cref{eq:two_time_denoiser}.
Then $\hat X$ is the flow map of $b_t$ if $\hat X_{s,s}={\rm Id}$ and it satisfies the tangent condition $\partial_t\hat X_{s,t}|_{t=s}=b_s$ and the semigroup condition $\hat X_{s,t}=\hat X_{u,t}\circ\hat X_{s,u}$~\citep{boffi2026build}.
We verify each.
$\hat X_{s,s}={\rm Id}$ holds by construction.
The tangent condition follows from the diagonal condition by differentiating \Cref{eq:two_time_denoiser} at $t=s$.
Substituting \Cref{eq:two_time_denoiser} twice, $\hat X_{u,t}\circ\hat X_{s,u}$ has the form of \Cref{eq:two_time_denoiser} with denoiser $\gamma^i\delta_{s,u}^i+(1-\gamma^i)\delta_{u,t}(\hat X_{s,u})^i$, so the semigroup condition on $\delta$ gives the semigroup condition on $\hat X$.
Hence $\hat X=X$.
\end{proof}

\paragraph{Distillation objective.}

By \Cref{prop:embedded_characterization}, learning the flow map reduces to satisfying the diagonal and semigroup conditions.
We impose both with KL losses on a student $\hat\delta_{s,t}$ initialized from the frozen teacher $\hat D_s$ (\Cref{sec:training}).

For $0\leq s<u<t<1$, the off-diagonal target composes two shorter student transitions \citep{lee2026flm},
\begin{equation}\label{eq:composition_target}
    \bar\delta_{s,t}({\bf x}_s;{\bf c})^i
    \coloneqq\mathsf{sg}\Big[
    \gamma^i\,\hat\delta_{s,u}({\bf x}_s;{\bf c})^i
    +(1-\gamma^i)\,\hat\delta_{u,t}({\bf x}_u;{\bf c})^i
    \Big],
    \qquad
    {\bf x}_u\coloneqq\hat X_{s,u}({\bf x}_s;{\bf c}),
\end{equation}
with $\gamma^i$ as in \Cref{eq:app_denoiser_semigroup}.
Since $0\leq\gamma^i\leq1$, the target is a categorical distribution.
The intermediate state ${\bf x}_u$ is computed by the student, and $\mathsf{sg}$ stops gradients through the entire target.
Since prefix positions are conditioned on time $1$, the student retains the prefix cache and an autoregressive predictor $f^{\hat{\delta}}$.
We train the student with semigroup-based flow map distillation~\citep{lee2026flm} plus autoregressive distillation:
\begin{equation}\label{eq:psd_loss}
\mathcal{L}_{\mathsf{distill}}(\hat{\delta})
\coloneqq
\underbrace{\mathbb{E}_{{\bf c}}\,\mathsf{KL}(f({\bf c})\|f^{\hat{\delta}}({\bf c}))}_{\text{AR distillation loss}}
+
\underbrace{\mathbb{E}_{t,{\bf c},I_t}\,\mathsf{KL}(\hat{D}_t(I_t;{\bf c})\|\hat{\delta}_{t,t}(I_t;{\bf c}))}_{\text{diagonal distillation loss}}
+
\underbrace{\mathbb{E}_{s,u,t,{\bf c},I_s}\,\mathsf{KL}(\bar\delta_{s,t}(I_s;{\bf c})\|\hat{\delta}_{s,t}(I_s;{\bf c}))}_{\text{off-diagonal distillation loss}},
\end{equation}
where each $\mathsf{KL}$ sums tokenwise over prefix or block positions.
The first two loss terms match the student's AR predictions and instantaneous denoising with the teacher, while the off-diagonal term enforces the semigroup property $\hat{X}_{s,t}\approx\hat{X}_{u,t}\circ\hat{X}_{s,u}$, bootstrapping long jumps from short ones up to one-step generation $\hat{X}_{0,1}$.

The following proposition shows that the two flow map terms of this objective are well posed, in that their population minimizers recover the teacher flow map.

\begin{proposition}
\label{prop:embedded_semigroup}
Assume full support of the sampling distributions over $(s,u,t)$ and the states.
Over all two-time denoisers $\delta$, the population objective
\begin{equation}\label{eq:embedded_semigroup_loss}
    \mathcal L(\delta)
    =\mathbb E_{t,I_t}\sum_{i=1}^B
    \mathsf{KL}\big(\hat{D}_t(I_t;{\bf c})^i\,\|\,\delta_{t,t}(I_t;{\bf c})^i\big)
    +\mathbb E_{s,u,t,I_s}\sum_{i=1}^B
    \mathsf{KL}\big(\bar\delta_{s,t}(I_s;{\bf c})^i\,\|\,\delta_{s,t}(I_s;{\bf c})^i\big)
\end{equation}
attains its minimum zero, and every population minimizer recovers the teacher flow map $X$ through \Cref{eq:two_time_denoiser}.
\end{proposition}

\begin{proof}
The denoiser constructed in \Cref{prop:embedded_characterization} satisfies both conditions, so both KL terms are zero.
Conversely, a zero loss implies the diagonal and semigroup conditions almost surely, and by full support and continuity, everywhere.
By the converse of \Cref{prop:embedded_characterization}, $\hat X=X$.
\end{proof}

Note that the two terms of \Cref{eq:embedded_semigroup_loss} are the diagonal and off-diagonal terms of \Cref{eq:psd_loss}. The autoregressive term depends only on the prefix readout and is zero at $f = f^{\hat\delta}$, so it does not change the population minimizers of the two flow map terms.

\section{Algorithmic details}
\label{app:implementation_details}

We provide training and inference pseudocode in \Cref{alg:block-flow-training,alg:block-time-sampler,alg:block-euler-sampling,alg:fused-denoiser-inference}, and provide the hyperparameters used for our main results in \Cref{tab:training-hyperparameters}.

\subsection{Training}
\label{app:training}
\label{app:time_reparam}

\begin{figure}[t!]
\centering
\begin{minipage}[t]{0.49\linewidth}
\begin{algorithm}[H]
\footnotesize
\DontPrintSemicolon
\SetNoFillComment
\SetKwComment{tcp}{$\triangleright$\ }{}
\caption{A step of parallel denoiser training}
\label{alg:block-flow-training}
\KwIn{Training sequence ${\bf y}\in V^{R\times B}$, block size $B$, transformer $f$ with embeddings ${\bf E}$, tokenwise time $\tau$, noise scale $\sigma$, self-conditioning probability $p_{\mathsf sc}$\;}
\BlankLine
$({\bf y}^{(1)}, ..., {\bf y}^{(R)})\gets {\bf y}$\tcp*[r]{$R$ blocks of $B$ clean tokens}
\For{each block $r$}{
  $t_r\sim\mathsf{U}[0,1]$\tcp*[r]{uniform base time}
  ${\bf x}_1^{(r)}\gets{\bf E}_{{\bf y}^{(r)}}$\tcp*[r]{data embedding}
  ${\bf x}_0^{(r)}\sim \mathsf{N}(0, \sigma^2I)$\tcp*[r]{noise}
  $I_{t_r}^{(r)}\gets (1-\tau(t_r)){\bf x}_0^{(r)} + \tau(t_r){\bf x}_1^{(r)} $\tcp*[r]{interpolant \Cref{eq:async_interpolant}}
}
$I \gets (I_{t_1}^{(1)}, ..., I_{t_R}^{(R)})$\tcp*[r]{$R$ blocks of $B$ noisy tokens}
$\boldsymbol{\tau} \gets (\tau(t_1), ..., \tau(t_R))$\tcp*[r]{$R$ blocks of $B$ tokenwise times}
\If{$\mathsf{U}[0,1] < p_{\mathsf sc}$}{
  Self-condition the noisy blocks following \Cref{eq:self_conditioning_training}\;
}
$({\bf q},{\bf p})\gets f_{({\bf 1}, \boldsymbol{\tau})}({\bf y},I)$\tcp*[r]{AR and denoising (\Cref{fig:training})}
${\mathcal L}(f)\gets\mathsf{CE}({\bf q},{\bf y})+\mathsf{CE}({\bf p},{\bf y})+|{\bf p}{\bf E}-{\bf E}_{{\bf y}}|^2$\tcp*[r]{loss \Cref{eq:loss}}
Update $f$ including ${\bf E}$ using $\mathcal L(f)$\;
Project each ${\bf E}_v$ to its pretrained norm\;
\KwRet{$f$}\;
\end{algorithm}
\end{minipage}\hfill
\begin{minipage}[t]{0.49\linewidth}
\begin{algorithm}[H]
\LinesNotNumbered
\footnotesize
\DontPrintSemicolon
\SetNoFillComment
\SetKwComment{tcp}{$\triangleright$\ }{}
\caption{A step of updating tokenwise time}
\label{alg:block-time-sampler}
\KwIn{Denoiser $D$ with embeddings ${\bf E}$, loss matrix ${\bf L}\in\mathbb R^{B\times M}$, block size $B$, time bins $M$, a batch of training sequences ${\bf y}$, noise scale $\sigma$, EMA decay $\beta$}
\BlankLine
$\bar{\bf L}\gets {\bf 0}\in\mathbb{R}^{B\times M}$\tcp*[r]{position-time bins}
\For{each sequence ${\bf y}$}{
  Sample a block ${\bf y}^{(r)}\in V^B$ from ${\bf y}$\;
  $t_r\sim{\mathsf{U}}[0,1]$\tcp*[r]{uniform base time}
  ${\bf x}_1^{(r)}\gets{\bf E}_{{\bf y}^{(r)}}$\tcp*[r]{embedding}
  ${\bf x}_0^{(r)}\sim \mathsf{N}(0, \sigma^2I)$\tcp*[r]{noise}
  $I_{t_r}^{(r)}\gets (1-t_r){\bf x}_0^{(r)} + t_r{\bf x}_1^{(r)} $\tcp*[r]{synced interpolant}
  ${\bf p}^{(r)}\gets{D_{t_r}}(I_{t_r}^{(r)};{\bf y}^{(1:r-1)})$\tcp*[r]{block denoising}
  $\bar{{\bf L}}\gets {\mathsf{record}}\!\left(\bar{{\bf L}}, \mathsf{CE}({\bf p}^{(r)},{\bf y}^{(r)}),t_r\right)$;\tcp*[r]{running mean}
}
${\bf L}_{:,m}\!\gets\!\beta{\bf L}_{:,m}\!+\!(1\!-\!\beta)\bar{{\bf L}}_{:,m}$ for each visited $m$\tcp*[r]{EMA \Cref{eq:time_sampler_ema}}
\For{each position $i=1, ..., B$}{
$F_i\gets{\mathsf{CDF}}(\frac{{\bf L}_{i,:}}{\sum_m{\bf L}_{i,m}})$\tcp*[r]{piecewise linear CDF}
$\tau^i\gets F_i^{-1}$\tcp*[r]{tokenwise time \Cref{eq:time_reparam}}
}
\KwRet{${\bf L},\tau$}\;
\end{algorithm}
\end{minipage}
\end{figure}

To train on many blocks at once, following prior work on block diffusion~\citep{arriola2025bd3lm,tian2025next,jin2026mars,liu2025wedlm}, we pair each clean training sequence with a copy which is split into noisy blocks of size $B$.
The attention mask applies prefix-causal and block-bidirectional structure to every block at once, such that clean positions attend causally, and each noisy block attends to all clean tokens preceding it and bidirectionally within itself, but not to later clean tokens or other noisy blocks.
Under this setting, each step of training progresses as in \Cref{alg:block-flow-training}.

During training, inspired by the loss-aware time warping of \citet{dieleman2022cdcd}, applied separately to each block position, we learn each tokenwise time as the inverse cumulative distribution of its denoising loss.
Let $\ell_i(\tau)>0$ be the expected denoising cross-entropy loss of position $i$ at tokenwise time $\tau$, and define
\begin{equation}\label{eq:time_reparam}
    \tau^i(t)\coloneqq F_i^{-1}(t),
    \qquad
    F_i(\tau)\coloneqq\frac{\int_0^\tau\ell_i(r)\,{\rm d}r}{\int_0^1\ell_i(r)\,{\rm d}r}.
\end{equation}
Since $\dot\tau^i\propto1/\ell_i(\tau^i)$ is small where the loss is high, both a uniform temporal grid over $t$ during inference and uniform sampling of $t$ during training concentrate on the hardest noise levels.

To estimate $F$ during training, we keep an exponential moving average ${\bf L}\in\mathbb{R}^{B\times M}$ of the per-position denoising cross entropy in $M$ uniform time bins, and set each $F_i$ to the piecewise-linear CDF of the normalized row ${\bf L}_{i,:}$.
We initialize ${\bf L}$ uniformly, and every $S\in\mathbb{N}$ optimization steps, run four gradient-free forward passes with independently sampled uniform times for each noisy block (see \Cref{alg:block-time-sampler} for details).
For every position-time bin $(i,m)$ encountered during these forward passes, we update
\begin{equation}\label{eq:time_sampler_ema}
    {\bf L}_{i,m}\gets\beta{\bf L}_{i,m}+(1-\beta)\bar{\bf L}_{i,m},
\end{equation}
where $\bar{\bf L}_{i,m}$ is the mean cross entropy recorded at position $i$ in time bin $m$, and bins that are not visited retain their previous value.
This independent uniform estimation prevents the learned sampling distribution from starving low-probability regions.

\begin{figure}[t!]
\centering
\begin{minipage}[t]{0.49\linewidth}

\begin{algorithm}[H]
\LinesNotNumbered
\footnotesize
\DontPrintSemicolon
\SetNoFillComment
\SetKwComment{tcp}{$\triangleright$\ }{}
\caption{Flow decoding}
\label{alg:block-euler-sampling}
\KwIn{Denoiser $D$ with embeddings ${\bf E}$, prefix ${\bf c}$, tokenwise time $\tau$, block size $B$, Euler steps $K$, guidance scale $g\geq1$, noise scale $\sigma$}
\BlankLine
$\hat{\bf x}_0\sim\mathsf N(0,\sigma^2I)$\tcp*[r]{initial noise}
$t_k\gets k/K$ for $k=0,\ldots,K$\tcp*[r]{base time grid}
${\bf s},{\bf s}'\gets\varnothing, \varnothing$\tcp*[r]{self-conditioning}
\For{each step $k=0, ..., K-1$}{
  ${\bf s}\gets D_{t_k}(\mathsf{SC}(\hat{\bf x}_k,{\bf s});{\bf c})$\tcp*[r]{usual forward}
  \eIf{$g>1$}{
    ${\bf s}'\gets D_{t_k}(\mathsf{SC}(\hat{\bf x}_k,{\bf s}'); \varnothing)$\tcp*[r]{prefix-free forward}
    $\tilde{\bf x}_k\gets{\bf s}'{\bf E}+g({\bf s}{\bf E}-{\bf s}'{\bf E})$\tcp*[r]{guided}
  }{
    $\tilde{\bf x}_k\gets{\bf s}{\bf E}$\tcp*[r]{unguided}
  }
  \For{each position $i=1, ..., B$}{
    $\hat{\bf x}_{k+1}^i\gets\hat{\bf x}_k^i+\frac{\tau^i(t_{k+1})-\tau^i(t_k)}{1-\tau^i(t_k)}(\tilde{\bf x}_k^i - \hat{\bf x}_k^i)$\tcp*[r]{step}
  }
}
$\hat{\mathbf y}\gets (\mathsf{argmax}({\bf s}^j) )_{j=1}^{B}$\tcp*[r]{decoding}
\KwRet{$\hat{\mathbf y}$}\;
\end{algorithm}
\end{minipage}\hfill
\begin{minipage}[t]{0.49\linewidth}
\begin{algorithm}[H]
\LinesNotNumbered
\footnotesize
\DontPrintSemicolon
\SetNoFillComment
\SetKwComment{tcp}{$\triangleright$\ }{}
\caption{Fused self-speculative decoding}
\label{alg:fused-denoiser-inference}
\KwIn{Denoiser $D$ and AR predictor $f$ sharing a transformer, prefix ${\bf c}$ with KV cache ${\bf h}$, block size $B$, scale $\sigma$}
\BlankLine
${\bf p}\gets D_0({\bf z};{\bf c})$, ${\bf z}\sim\mathsf N(0,\sigma^2I)$, ${\bf q}\gets f({\bf c})$\tcp*[r]{draft \& AR prediction}
\While{the stopping condition is not met}{
  $\hat{\bf y}\gets(\mathsf{argmax}({\bf p}^i))_{i=1}^B$\tcp*[r]{rounding}
  \eIf{$\hat{\bf y}^1\neq\mathsf{argmax}({\bf q}^{|c|})$}{
    $\hat{\bf y}^1\gets\mathsf{argmax}({\bf q}^{|c|})$, $A\gets\{1\}$\tcp*[r]{correction}
  }{
    $A\gets[B]$\tcp*[r]{all possible acceptance lengths}
  }
  ${\bf z}^{(j)}\sim \mathsf{N}(0, \sigma^2I)$ for $j\in A$\tcp*[r]{noise blocks}
  Run a single forward pass on prefix ${\bf c}\oplus \hat{\bf y}^{1:\max A}$ and blocks $({\bf z}^{(j)})_{j\in A}$ with KV cache ${\bf h}$\tcp*[r]{(\Cref{fig:fused_decoding})}
  $({\bf q}; {\bf h})\gets f({\bf c}\oplus \hat{\bf y}^{1:\max A})$\tcp*[r]{get AR prediction and KV}
  ${\bf p}^{(j) }\gets D_0({\bf z}^{(j)}; {\bf c}\oplus \hat{\bf y}^{1:j})$ for $j\in A$\tcp*[r]{get drafts}
  $a\gets\max\{j\in A:\hat{\bf y}^{i+1}=\mathsf{argmax}({\bf q}^{|c|+i})\ \forall\,1\leq i<j\}$\;
  $({\bf c},\!{\bf h},\!{\bf q},\!{\bf p}) \gets ({\bf c}\!\oplus\!\hat{\bf y}^{1:a}, {\bf h}^{:|c|+a}, {\bf q}^{:|c|+a}, {\bf p}^{(a)})$\tcp*[r]{update}
}
\KwRet{${\bf c}$}\;
\end{algorithm}
\end{minipage}
\end{figure}

\subsection{Flow decoding}
\label{app:self_conditioning}
\label{app:prefix_guidance}

To improve flow decoding, we apply self-conditioning during training with a fixed probability by obtaining an initial prediction without gradient and using its expected token embeddings to condition a second forward pass.
Following \citet{chen2026langflow}, we incorporate this prediction through additional projection layers:
\begin{equation}\label{eq:self_conditioning_training}
    \bar{\bf x}_1= \mathsf{sg}\big(D_t(I_t;{\bf c}){\bf E}\big), \qquad \tilde I_t= I_t + I_t{\bf W}_1^\top+ \bar{\bf x}_1{\bf W}_2^\top,
\end{equation}
where ${\bf W}_1,{\bf W}_2\in\mathbb{R}^{d\times d}$ are trainable and initialized to zero.
For training under self-conditioning, we replace the noisy block input with $\tilde I_t$ and compute the training loss on the second pass.
Otherwise, we train with a single pass on $I_t$.
In practice, we use a probability $p_\mathsf{sc}=0.5$ of using self-conditioning for each training batch.

During inference, we use the denoiser prediction from the previous integration step for self-conditioning.
In single-step update of Euler integration of \Cref{eq:block_flow_ode}, we replace $D_{t_k}(\hat{\bf x}_{t_k};{\bf c})$ with ${\bf q}_k$, defined by
\begin{equation}
    \tilde{\bf x}_{t_k} = \hat{\bf x}_{t_k}+ \hat{\bf x}_{t_k}{\bf W}_1^\top + ({\bf q}_{k-1}{\bf E}){\bf W}_2^\top, \qquad {\bf q}_k = D_{t_k}(\tilde{\bf x}_{t_k};{\bf c}),\qquad {\bf q}_0=D_0(\hat{\bf x}_0;{\bf c}).
\end{equation}
Detailed procedure for flow decoding can be found in \Cref{alg:block-euler-sampling}, where we write, for compactness:
\begin{equation}
\mathsf{SC}({\bf x},{\bf q}) \coloneqq {\bf x}+{\bf x}{\bf W}_1^\top+({\bf q}{\bf E}){\bf W}_2^\top,
\qquad
\mathsf{SC}({\bf x},\varnothing) \coloneqq {\bf x}.
\end{equation}
In addition to self-conditioning, during flow decoding, we employ prefix guidance to strengthen the conditioning on the prefix, analogously to history guidance in video diffusion models~\citep{song2025history}.
We define the guided denoiser as
\begin{equation}\label{eq:prefix_guidance}
    D_t^g({\bf x};{\bf c})
    \coloneqq D_t({\bf x};\varnothing)
    +g(D_t({\bf x};{\bf c})
    -D_t({\bf x};\varnothing)),
    \qquad g\geq1,
\end{equation}
and replace $D_{t_k}(\hat{\bf x}_{t_k};{\bf c}){\bf E}$ in (self-conditioned) Euler integration \Cref{eq:block_flow_ode} with $D_{t_k}^g(\hat{\bf x}_{t_k};{\bf c}){\bf E}$.
Guidance scale $g=1$ recovers the usual prediction, whereas $g>1$ changes the velocity and hence the sampled distribution.

\subsection{Self-speculative decoding}
\label{app:denoiser_inference}

\begin{figure}[t!]
\centering
\includegraphics[width=\linewidth]{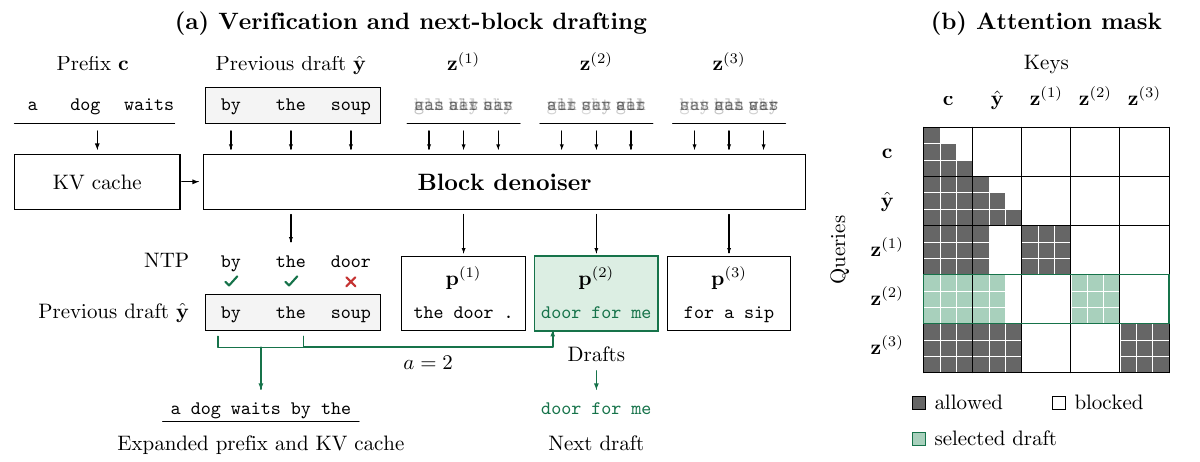}
\caption{Fused self-speculative decoding with block size three.
A single forward pass verifies previous draft with AR prediction and performs next-block drafting.
This is done by decoding next drafts corresponding to every possible acceptance of the previous draft in parallel and selecting one based on verification.}
\label{fig:fused_decoding}
\end{figure}

\begin{table}[t!]
\centering
\footnotesize
\caption{
Main hyperparameters.
For block size four, we use batches of 384 sequences for training, whereas for larger blocks, we use memory-matched microbatches and, when applicable, continue training from smaller-block checkpoint.
}
\label{tab:training-hyperparameters}
\begin{tabular}{lccc}
\toprule
{Setting} & {\ours{} 0.5B} & {\ours{} 7B} & {\ours{} 8B} \\
\midrule
\rowcolor{gray!10}\multicolumn{4}{c}{\textit{Model and data}} \\
Initialization & Qwen2.5 0.5B Instruct & Qwen2.5 7B Instruct & Qwen3 8B Base \\
Training data / epochs / tokens & \multicolumn{3}{c}{Dolci-Instruct-SFT / 5 epochs / approximately 9B} \\
Block sizes $B$ & 4, 8, 16 & 4, 8, 16 & 4 \\
Maximum sequence length & 512--1024 & 512--1024 & 512 \\
adaLN & Per-layer & Shared across layers & Shared across layers \\
\midrule
\rowcolor{gray!10}\multicolumn{4}{c}{\textit{Optimization}} \\
Optimizer & \multicolumn{3}{c}{AdamW} \\
Learning rate & $1{\times}10^{-4}$ & $1{\times}10^{-5}$ & $1{\times}10^{-5}$ \\
Schedule / warm-up & \multicolumn{3}{c}{Constant / first 1\% of steps} \\
Adam $(\beta_1,\beta_2)$ & \multicolumn{3}{c}{(0.9, 0.999)} \\
Weight decay / gradient clipping / batch size & \multicolumn{3}{c}{0 / 1.0 / 384 sequences} \\
Precision & \multicolumn{3}{c}{\texttt{bfloat16}} \\
Distributed training & DDP & ZeRO-3 & ZeRO-3 \\
\midrule
\rowcolor{gray!10}\multicolumn{4}{c}{\textit{Training objective and architecture}} \\
Training noise scale $\sigma$ & \multicolumn{3}{c}{4.0} \\
Self-conditioning probability $p_\mathsf{sc}$ & \multicolumn{3}{c}{0.5} \\
adaLN learning rate multiplier & \multicolumn{3}{c}{5.0} \\
Parameter EMA decay & \multicolumn{3}{c}{0.999 (enabled from step 500)} \\
\midrule
\rowcolor{gray!10}\multicolumn{4}{c}{\textit{Learned tokenwise time}} \\
Histogram bins $M$ & \multicolumn{3}{c}{200} \\
Histogram EMA decay $\beta$ & \multicolumn{3}{c}{0.99} \\
Update interval $S$ & \multicolumn{3}{c}{100 steps} \\
\midrule
\rowcolor{gray!10}\multicolumn{4}{c}{\textit{Flow decoding}} \\
Inference noise scale $\sigma$ & \multicolumn{3}{c}{0.1} \\
prefix guidance scale $g$ & \multicolumn{3}{c}{8.0} \\
\bottomrule
\end{tabular}
\end{table}

Self-speculative decoding drafts a block using the denoiser $D$ and commits its longest accepted prefix using AR verification with $f$.
We find it possible to build a fused operation that uses a single forward pass to simultaneously perform verification of a previously drafted block and drafting of the next block.
%

In each step of the fused decoding, we are given a prefix ${\bf c}\in V^N$, the AR prediction on the prefix ${\bf q} \coloneqq f({\bf c})\in(\Delta^{|V|-1})^N$, and a previously drafted block ${\bf p}\in (\Delta^{|V|-1})^B$ with $\hat{\bf y} = \mathsf{argmax}({\bf p}^i)_{i=1}^B \in V^B$.
We would like to simultaneously accept a prefix $\hat{\bf y}^{1:a}$ using AR verification with $f$, and also draft the next block via $D_0({\bf z};{\bf c} \oplus \hat{\bf y}^{1:a})$ where ${\bf z}$ denotes noise and $\oplus$ denotes concatenation.

Since we have the AR prediction ${\bf q}$ on the prefix, the verification for the first token $\hat{\bf y}^1$ can be done for free by comparing against ${\bf q}^N$.
If $\hat{\bf y}^1 \neq \mathsf{argmax}({\bf q}^N)$, then we simply replace $\hat{\bf y}^1$ by $\mathsf{argmax}({\bf q}^N)$ and accept $a=1$.
Drafting then evaluates $D_0({\bf z}; {\bf c}\oplus \hat{\bf y}^1)$, which simultaneously produces the AR prediction $f({\bf c}\oplus \hat{\bf y}^1)$ on the extended prefix.

When $\hat{\bf y}^1 = \mathsf{argmax}({\bf q}^N)$, we need block verification.
To perform this verification and next-block drafting in a single forward pass, we construct the input to the transformer underlying $D$ and $f$ as a concatenation of a clean prefix ${\bf c} \oplus \hat{\bf y}$ and $B$ blocks of $B$ noise tokens, ${\bf z}^{(1)}, ..., {\bf z}^{(B)}\in\mathbb{R}^{B\times d}$, each drawn from $\mathsf{N}(0, \sigma^2I)$.
We then construct the attention mask such that the prefix ${\bf c} \oplus \hat{\bf y}$ attends causally to itself, while each block ${\bf z}^{(j)}$ attends bidirectionally to itself and causally to a prefix of the prefix, ${\bf c} \oplus \hat{\bf y}^{1:j}$ (\Cref{fig:fused_decoding}).
Accordingly, we build tokenwise time such that ${\bf c} \oplus \hat{\bf y}$ is conditioned on one and all ${\bf z}^{(j)}$ on zero.
Since this produces the AR prediction $f({\bf c} \oplus \hat{\bf y})$ on the prefix, we can determine acceptance $\hat{\bf y}^{1:a}$ as follows (\Cref{sec:inference_schemes}):
\begin{equation}
    a = \max\{j\in [B]: \hat{\bf y}^{i+1} = \mathsf{argmax}(f({\bf c} \oplus \hat{\bf y})^{N+i})\text{ for every } 1\leq i< j\}.
\end{equation}
Then, instead of an additional forward pass, we simply take output tokens corresponding to ${\bf z}^{(a)}$ since their production is computationally identical to $D_0({\bf z}^{(a)}; {\bf c}\oplus \hat{\bf y}^{1:a})$.
The next step of fused decoding then operates on extended prefix ${\bf c}\oplus\hat{\bf y}^{1:a}$, drafted block $D_0({\bf z}^{(a)}; {\bf c}\oplus \hat{\bf y}^{1:a})$, and prefix AR prediction $f({\bf c} \oplus \hat{\bf y}^{1:a}) = f({\bf c} \oplus \hat{\bf y})^{1:N+a}$.
In practice, we further reuse the key-value cache of the prefix to improve efficiency (\Cref{alg:fused-denoiser-inference}).

\section{Evaluation details}
\label{app:eval_details}

\begin{table}[t!]
\centering
\footnotesize
\caption{Flow and flow map decoding accuracies.}
\label{tab:flow_full}
\resizebox{\linewidth}{!}{%
\begin{tabular}{lcccccccccc}
\toprule
\rowcolor{gray!10}
\multicolumn{3}{c}{} & \multicolumn{4}{c}{\emph{0.5B scale}} & \multicolumn{4}{c}{\emph{7B scale}} \\
\cmidrule(lr){4-7}\cmidrule(lr){8-11}
{Method} & $B$ & {Euler steps}
  & {GSM8K} & {HumanEval} & {IFEval} & {Average}
  & {GSM8K} & {HumanEval} & {IFEval} & {Average} \\
\midrule
AR SFT & -- & -- & \textbf{32.0} & \textbf{35.4} & 48.4 & \textbf{38.6} & 69.1 & \textbf{78.7} & \underline{73.8} & 73.9 \\
\midrule
\rowcolor{blue!5} \multicolumn{11}{c}{\emph{Flow decoding}} \\
\multirow{24}{*}{\textbf{\ours{} (Ours)}} & \multirow{8}{*}{16} & 1 & 0.8 & 3.0 & 20.3 & 8.1 & 2.4 & 10.4 & 29.8 & 14.2 \\
 &  & 2 & 1.7 & 6.1 & 23.5 & 10.4 & 18.7 & 21.3 & 36.2 & 25.4 \\
 &  & 4 & 11.7 & 12.2 & 27.0 & 17.0 & 47.5 & 40.2 & 44.7 & 44.1 \\
 &  & 8 & 18.3 & 14.0 & 30.9 & 21.1 & 60.2 & 54.9 & 56.9 & 57.3 \\
 &  & 16 & 24.8 & 17.1 & 37.2 & 26.3 & 65.6 & 65.2 & 64.9 & 65.2 \\
 &  & 32 & 25.9 & 21.3 & 39.9 & 29.1 & 66.9 & 68.3 & 64.1 & 66.5 \\
 &  & 64 & 28.1 & 21.3 & 42.3 & 30.6 & 67.5 & 66.5 & 66.5 & 66.8 \\
 &  & 128 & 28.1 & 20.7 & 44.2 & 31.0 & 68.7 & 67.7 & 65.2 & 67.2 \\
\cmidrule{2-11}
 & \multirow{8}{*}{8} & 1 & 2.4 & 6.1 & 23.3 & 10.6 & 20.1 & 22.0 & 34.8 & 25.6 \\
 &  & 2 & 9.2 & 12.2 & 23.5 & 15.0 & 47.4 & 41.5 & 42.1 & 43.7 \\
 &  & 4 & 17.4 & 18.9 & 29.0 & 21.8 & 64.0 & 57.9 & 55.8 & 59.2 \\
 &  & 8 & 24.8 & 25.6 & 38.4 & 29.6 & 64.8 & 64.0 & 64.5 & 64.5 \\
 &  & 16 & 28.4 & 28.0 & 46.4 & 34.3 & 71.6 & 70.7 & 64.5 & 68.9 \\
 &  & 32 & 29.1 & 28.0 & 45.5 & 34.2 & 72.3 & 72.0 & 71.0 & 71.7 \\
 &  & 64 & 28.4 & 28.0 & 47.1 & 34.5 & 71.7 & 72.6 & 68.8 & 71.0 \\
 &  & 128 & 28.7 & 28.7 & 49.0 & 35.5 & 71.0 & 72.6 & 69.3 & 70.9 \\
\cmidrule{2-11}
 & & 1 & 16.4 & 13.4 & 32.5 & 20.8 & 53.9 & 32.3 & 44.7 & 43.7 \\
 &  & 2 & 14.7 & 9.1 & 35.1 & 19.7 & 60.0 & 34.8 & 43.8 & 46.2 \\
 &  & 4 & 25.0 & 21.3 & 43.6 & 30.0 & 69.5 & 63.4 & 60.1 & 64.3 \\
 &  & 8 & 27.0 & 22.0 & 46.6 & 31.8 & 73.0 & 71.3 & 69.3 & 71.2 \\
 &  & 16 & 28.8 & 27.4 & 50.1 & 35.4 & 71.9 & 75.0 & 69.5 & 72.1 \\
 & & 32 & \underline{30.4} & 31.1 & \underline{52.9} & 38.1 & 73.8 & \textbf{78.7} & 73.0 & 75.2 \\
 & & 64 & 29.6 & \underline{32.9} & \textbf{53.0} & \underline{38.5} & \underline{74.6} & \textbf{78.7} & \textbf{74.3} & \textbf{75.9} \\
 & \multirow{-8}{*}{4} & 128 & 30.1 & 31.7 & 49.4 & 37.1 & \textbf{75.6} & \underline{78.0} & 73.6 & \underline{75.7} \\
\midrule
\rowcolor{blue!5} \multicolumn{11}{c}{\emph{Flow map decoding}} \\
 & & 1 & 19.2 & 17.1 & 33.6 & 23.3 & 54.2 & 50.0 & 42.7 & 49.0 \\
 & & 2 & 21.1 & 18.3 & 36.2 & 25.2 & 66.4 & 68.3 & 55.1 & 63.3 \\
 & & 4 & 24.9 & 21.3 & 39.4 & 28.5 & 69.7 & 73.8 & 64.7 & 69.4 \\
 & & 8 & 28.1 & 26.8 & 40.3 & 31.8 & 70.9 & 73.8 & 68.2 & 71.0 \\
 \multirow{-5}{*}{\textbf{\ours{} (Ours)}} & \multirow{-5}{*}{4} & 16 & 28.9 & 25.6 & 43.3 & 32.6 & -- & -- & -- & -- \\
\bottomrule
\end{tabular}}
\end{table}

\begin{table}[t!]
\centering
\footnotesize
\caption{Single-token decoding.
For diffusion models, we decode exactly one highest-probability token per forward.
Models marked $^\dagger$ are trained by us from the matched AR model on Dolci-Instruct-SFT for five epochs.}
\label{tab:ntp_full}
\begin{tabular}{lccccccc}
\toprule
{Method} & {IFEval} & {BBH} & {MMLU-Pro} & {GPQA} & {GSM8K} & {HumanEval} & {Average} \\
\midrule
\rowcolor{gray!10} \multicolumn{8}{c}{\emph{0.5B scale}} \\
Qwen2.5 SFT$^\dagger$ & 48.4 & 26.3 & 11.9 & 17.9 & 32.0 & 35.4 & 28.7 \\
MTP$^\dagger$~\citep{gloeckle2024better} & 50.6 & 17.3 & 9.2 & 18.8 & 27.1 & 32.3 & 25.9 \\
DeepSeek-MTP$^\dagger$~\citep{deepseekv3} & 53.2 & 24.1 & \textbf{14.0} & 18.5 & \underline{35.5} & 36.0 & 30.2 \\
MARS$^\dagger$~\citep{jin2026mars} & 51.3 & \underline{26.6} & \underline{12.4} & 19.4 & 32.8 & \textbf{40.2} & \underline{30.4} \\
SBD$^\dagger$~\citep{gat2025set} & \textbf{55.3} & 20.6 & 9.2 & \textbf{22.5} & 27.2 & 30.8 & 27.6 \\
BD3LM$^\dagger$~\citep{arriola2025bd3lm} & 47.1 & 7.5 & 2.0 & 17.9 & 30.6 & 31.7 & 22.8 \\
\rowcolor{blue!5} \textbf{\ours{} 0.5B (Ours)$^\dagger$} & \underline{55.1} & \textbf{27.6} & 12.1 & \underline{21.9} & \textbf{36.1} & \underline{36.6} & \textbf{31.6} \\
\midrule
\rowcolor{gray!10} \multicolumn{8}{c}{\emph{7B scale}} \\
Qwen2.5 SFT$^\dagger$ & 73.8 & 59.2 & 41.6 & 27.9 & 69.1 & \underline{78.7} & 58.4 \\
MTP$^\dagger$~\citep{gloeckle2024better} & 73.2 & 51.7 & 37.3 & 26.8 & 60.5 & 73.8 & 53.9 \\
DeepSeek-MTP$^\dagger$~\citep{deepseekv3} & \underline{78.0} & 60.7 & 40.6 & 30.1 & 70.3 & 80.5 & \underline{60.0} \\
MARS$^\dagger$~\citep{jin2026mars} & 77.8 & 54.6 & \textbf{44.4} & 26.6 & 73.2 & \textbf{81.7} & 59.7 \\
Fast-dLLM~v2$^\dagger$ & 70.6 & 53.5 & \underline{43.1} & \underline{28.6} & 67.2 & 32.9 & 49.3 \\
Fast-dLLM~v2~\citep{wu2026fast} & 64.5 & 58.1 & 39.7 & 21.4 & \underline{80.3} & 62.2 & 54.4 \\
Dream~\citep{dream2024} & 57.1 & \textbf{66.0} & 43.6 & 26.6 & \textbf{81.7} & 58.5 & 55.6 \\
\rowcolor{blue!5} \textbf{\ours{} (Ours)$^\dagger$} & \textbf{78.9} & \underline{63.0} & 41.8 & \textbf{31.0} & 73.8 & 76.8 & \textbf{60.9} \\
\midrule
\rowcolor{gray!10} \multicolumn{8}{c}{\emph{8B scale}} \\
Qwen3 SFT$^\dagger$ & 38.1 & \textbf{71.6} & \underline{56.2} & 29.5 & 62.9 & \textbf{82.3} & 56.8 \\
LLaDA~\citep{llada2025} & 58.0 & 30.9 & 28.7 & 25.4 & 76.0 & 42.1 & 43.5 \\
Efficient-DLM~\citep{fu2026edlm}
 & \underline{58.4} & \underline{70.8} & \textbf{61.5} & \textbf{33.7} & \textbf{87.4} & 68.9 & \textbf{63.5} \\
\rowcolor{blue!5} \textbf{\ours{} (Ours)$^\dagger$} & \textbf{62.3} & 70.7 & \underline{56.2} & \underline{31.7} & \underline{77.8} & \underline{79.9} & \underline{63.1} \\
\bottomrule
\end{tabular}
\end{table}

\begin{table}[t!]
\centering
\footnotesize
\caption{Self-speculative and parallel decoding accuracy (\%) and TPF across model scales.}
\label{tab:selfspec_full}
\resizebox{\linewidth}{!}{%
\begin{tabular}{lccccccccccccccc}
\toprule
 &  & \multicolumn{2}{c}{IFEval} & \multicolumn{2}{c}{BBH} & \multicolumn{2}{c}{MMLU-Pro} & \multicolumn{2}{c}{GPQA} & \multicolumn{2}{c}{GSM8K} & \multicolumn{2}{c}{HumanEval} & \multicolumn{2}{c}{Average} \\
\cmidrule(lr){3-4}\cmidrule(lr){5-6}\cmidrule(lr){7-8}\cmidrule(lr){9-10}\cmidrule(lr){11-12}\cmidrule(lr){13-14}\cmidrule(lr){15-16}
Method & $B$ & Acc. & TPF & Acc. & TPF & Acc. & TPF & Acc. & TPF & Acc. & TPF & Acc. & TPF & Acc. & TPF \\
\midrule
\rowcolor{gray!10} \multicolumn{16}{c}{\emph{0.5B scale}} \\
AR SFT & -- & 48.4 & 1.0 & 26.3 & 1.0 & 11.9 & 1.0 & 17.9 & 1.0 & 32.0 & 1.0 & 35.4 & 1.0 & 28.7 & 1.0 \\
MARS & 4 & 51.0 & 1.8 & \textbf{31.7} & 2.3 & \textbf{13.9} & 1.7 & 18.3 & 2.0 & 31.7 & 2.1 & 32.3 & 1.9 & 29.8 & 2.0 \\
MTP & 4 & 51.2 & 2.5 & 17.6 & 3.3 & 9.4 & 3.1 & 17.9 & 2.8 & 26.8 & 2.9 & 33.5 & 2.9 & 26.1 & 2.9 \\
DeepSeek-MTP & 2 & 51.9 & 1.6 & 23.7 & 1.8 & \underline{12.1} & 1.9 & 18.8 & 1.9 & 35.5 & 1.8 & 35.4 & 1.7 & 29.6 & 1.8 \\
BD3LM & 4 & 46.8 & 1.5 & 8.3 & 1.6 & 1.3 & 1.5 & 19.4 & 1.5 & 29.8 & 1.5 & 29.9 & 1.5 & 22.6 & 1.5 \\
SBD & 4 & \textbf{56.4} & 2.2 & 20.6 & 3.2 & 9.4 & 2.7 & \underline{23.2} & 2.7 & 26.7 & 2.7 & 31.1 & 2.8 & 27.9 & 2.7 \\
PTP & 4 & 42.3 & 1.7 & 26.0 & 2.5 & 11.9 & 2.0 & 15.0 & 1.9 & 35.3 & 2.3 & \textbf{42.7} & 2.1 & 28.9 & 2.1 \\
\rowcolor{blue!5} & 4 & 55.1 & 2.2 & 27.6 & 3.3 & \underline{12.1} & 3.1 & 21.9 & 2.9 & 36.1 & 2.7 & 36.6 & 2.9 & \underline{31.6} & 2.9 \\
\rowcolor{blue!5} & 8 & \underline{55.3} & 2.7 & \underline{27.9} & 5.1 & 10.0 & 5.5 & \textbf{23.7} & 4.2 & \textbf{37.8} & 3.6 & 39.0 & 3.0 & \textbf{32.3} & 4.0 \\
\rowcolor{blue!5} \multirow{-3}{*}{\textbf{\ours{} (Ours)}} & 16 & 54.2 & 3.1 & 27.0 & 6.3 & 9.1 & 9.2 & 18.1 & 4.8 & \underline{36.7} & 3.6 & \underline{39.6} & 3.3 & 30.8 & 5.0 \\
\midrule
\rowcolor{gray!10} \multicolumn{16}{c}{\emph{7B scale}} \\
AR SFT & -- & 73.8 & 1.0 & 59.2 & 1.0 & 41.6 & 1.0 & 27.9 & 1.0 & 69.1 & 1.0 & 78.7 & 1.0 & 58.4 & 1.0 \\
MTP & 4 & 73.2 & 1.9 & 51.9 & 3.3 & 38.3 & 2.5 & 27.7 & 2.6 & 60.5 & 3.0 & 73.8 & 3.1 & 54.2 & 2.7 \\
DeepSeek-MTP & 2 & \underline{78.0} & 1.5 & \underline{60.8} & 1.8 & \textbf{42.5} & 1.6 & 29.5 & 1.7 & 70.3 & 1.7 & 80.5 & 1.9 & \underline{60.3} & 1.7 \\
Fast-dLLM~v2 & 32/8 & 62.8 & 1.9 & 58.1 & 2.3 & 39.2 & 2.2 & 22.3 & 2.2 & \textbf{80.4} & 2.5 & 63.4 & 1.8 & 54.4 & 2.1 \\
Dream & Canvas & 18.9 & 4.2 & 42.5 & 4.2 & \underline{42.2} & 4.4 & 9.8 & 4.4 & 38.4 & 3.5 & 23.8 & 3.8 & 29.3 & 4.1 \\
 & 4 & 77.8 & 1.9 & 52.9 & 4.1 & 40.0 & 2.8 & 29.2 & 2.8 & 74.0 & 2.8 & \textbf{81.7} & 2.9 & 59.3 & 2.9 \\
MARS & 8 & 76.3 & 2.0 & 59.6 & 5.7 & 41.5 & 3.3 & \underline{30.8} & 3.2 & 71.3 & 3.5 & \underline{81.1} & 3.5 & 60.1 & 3.5 \\
 & 16 & 75.8 & 2.0 & 55.8 & 6.3 & 39.2 & 3.6 & 26.8 & 3.2 & \underline{75.4} & 3.7 & 77.4 & 3.7 & 58.4 & 3.7 \\
\rowcolor{blue!5} & 4 & \textbf{78.9} & 1.7 & \textbf{63.0} & 3.4 & 41.8 & 2.4 & \textbf{31.0} & 2.3 & 73.8 & 2.8 & 76.8 & 2.8 & \textbf{60.9} & 2.6 \\
\rowcolor{blue!5} & 8 & 77.5 & 2.8 & 59.1 & 6.3 & 39.4 & 3.7 & 28.1 & 4.0 & 75.3 & 4.8 & 76.2 & 4.6 & 59.3 & 4.3 \\
\rowcolor{blue!5} \multirow{-3}{*}{\textbf{\ours{} (Ours)}} & 16 & 74.9 & 2.8 & 57.3 & 7.9 & 39.0 & 3.8 & 28.1 & 4.0 & 75.0 & 5.8 & 73.8 & 5.4 & 58.0 & 4.9 \\
\midrule
\rowcolor{gray!10} \multicolumn{16}{c}{\emph{8B scale}} \\
AR SFT & -- & 38.1 & 1.0 & \textbf{71.6} & 1.0 & 56.2 & 1.0 & 29.5 & 1.0 & 62.9 & 1.0 & \textbf{82.3} & 1.0 & 56.8 & 1.0 \\
LLaDA & Canvas & 36.8 & 3.2 & 31.0 & 3.2 & 28.8 & 2.7 & 9.4 & 2.7 & 47.2 & 2.7 & 23.8 & 2.5 & 29.5 & 2.8 \\
 & 4 & 57.1 & 2.6 & 66.9 & 3.2 & \underline{60.7} & 3.0 & \underline{33.0} & 2.8 & \underline{85.4} & 2.9 & 66.5 & 3.0 & 61.6 & 2.9 \\
\multirow{-2}{*}{Efficient-DLM} & 32 & \underline{58.8} & 2.4 & \underline{70.7} & 3.3 & \textbf{61.6} & 3.1 & \textbf{34.4} & 2.9 & \textbf{87.4} & 2.9 & 68.9 & 3.0 & \textbf{63.6} & 2.9 \\
\rowcolor{blue!5} \textbf{\ours{} (Ours)} & 4 & \textbf{62.3} & 2.4 & \underline{70.7} & 3.4 & 56.2 & 3.2 & 31.7 & 3.1 & 77.8 & 2.6 & \underline{79.9} & 2.9 & \underline{63.1} & 2.9 \\
\bottomrule
\end{tabular}}
\end{table}

We describe the evaluation setup and metrics, then report detailed results including per-benchmark scores.

\subsection{TPF evaluation and per-benchmark scores}
\label{app:full_results}

\paragraph{Tokens per forward (TPF).}
Across evaluation examples $j$, we compute TPF as
\begin{equation}\label{eq:tokens_per_forward}
    \text{TPF} = \frac{\sum_j T_j}{\sum_j F_j},
\end{equation}
where $T_j$ is the number of committed tokens and $F_j$ is the number of forward passes.
AR has TPF one by definition.
For \ours{} and Efficient-DLM, $F_j$ counts the number of fused verification-and-redraft forward passes in addition to one forward pass for initial drafting, excluding prompt prefill and cache refreshes for Efficient-DLM.
For Fast-dLLM~v2, $F_j$ includes top-level denoising, prefill, and cache update forward passes.

\paragraph{Per-benchmark scores.}
\Cref{tab:flow_full} gives the per-benchmark scores of flow and flow map decoding used to construct \Cref{fig:multitoken_frontier}.
\Cref{tab:ntp_full} expands single-token decoding \Cref{tab:ntp_main} into per-benchmark accuracies.
\Cref{tab:selfspec_full} gives the per-benchmark values underlying the parallel decoding results in \Cref{fig:selfspec_frontier} including our self-speculative decoding results.
All methods in \Cref{fig:selfspec_frontier} decode by drafting and verifying, and the verifier follows the native configuration of each method.
Causal block predictors verify against greedy AR, MARS uses its released batched verifier, and Efficient-DLM uses its exact verifier.
Fast-dLLM~v2, Dream, and LLaDA instead verify against their deterministic one-position decoding trajectories.

\subsection{Decoding throughput}
\label{app:throughput}

\begin{table}[t!]
\centering
\footnotesize
\caption{Per-task accuracy (\%) and throughput in tokens per second (TPS) at 7--8B scales, measured using one H100 GPU with batch size one.
Efficient-DLM is 8B and other methods are 7B.}
\label{tab:throughput_full}
\begin{tabular}{llccccccccc}
\toprule
 &  &  & \multicolumn{2}{c}{IFEval} & \multicolumn{2}{c}{GSM8K} & \multicolumn{2}{c}{HumanEval} & \multicolumn{2}{c}{Average} \\
\cmidrule(lr){4-5}\cmidrule(lr){6-7}\cmidrule(lr){8-9}\cmidrule(lr){10-11}
Method & Setting & Threshold & Accuracy & TPS & Accuracy & TPS & Accuracy & TPS & Accuracy & TPS \\
\midrule
AR SFT & -- & -- & 73.8 & 40.9 & 69.1 & 40.6 & 78.7 & 37.0 & 73.9 & 39.5 \\
DeepSeek-MTP & $B=2$ & -- & \underline{78.0} & 45.1 & 70.3 & 57.2 & \textbf{80.5} & 58.4 & \underline{76.3} & 53.6 \\
\midrule
\multirow{3}{*}{MARS} & $B=4$ & 0.95 & 77.6 & 42.0 & 70.7 & 59.8 & \underline{79.9} & 43.3 & 76.1 & 48.4 \\
 & $B=8$ & 0.95 & 77.3 & 40.9 & 69.1 & 62.7 & 78.0 & 43.1 & 74.8 & 48.9 \\
 & $B=16$ & 0.95 & 75.0 & 41.1 & 72.9 & 61.3 & 75.0 & 44.6 & 74.3 & 49.0 \\
\midrule
\multirow{3}{*}{Fast-dLLM~v2} & $B=4$ & 0.90 & 65.8 & 52.1 & 81.6 & 66.2 & 67.1 & 56.5 & 71.5 & 58.3 \\
 & $B=8$ & 0.90 & 65.4 & 60.9 & 80.1 & 98.5 & 57.3 & 71.1 & 67.6 & 76.8 \\
 & $B=16$ & 0.90 & 61.9 & 65.1 & 77.3 & 116.0 & 56.1 & 78.9 & 65.1 & 86.7 \\
\midrule
\multirow{6}{*}{Efficient-DLM} & $B=4$ & 0.70 & 58.2 & 31.8 & \underline{87.8} & 37.1 & 70.1 & 38.4 & 72.0 & 35.8 \\
 & $B=8$ & 0.70 & 58.2 & 40.7 & 87.7 & 55.6 & 68.9 & 58.5 & 71.6 & 51.6 \\
 & $B=16$ & 0.70 & 58.4 & 47.2 & \textbf{88.4} & 70.5 & 67.1 & 75.4 & 71.3 & 64.4 \\
 & $B=32$ & 0.70 & 59.1 & 49.3 & 87.5 & 74.4 & 63.4 & 84.8 & 70.0 & 69.5 \\
 & $B=32$ & 0.60 & 58.6 & 56.1 & 82.2 & 87.1 & 55.5 & 96.3 & 65.4 & 79.9 \\
 & $B=32$ & 0.50 & 52.9 & 71.9 & 72.0 & 102.9 & 42.1 & 111.2 & 55.7 & 95.3 \\
\midrule
\multirow{3}{*}{Dream} & $k=1$ & -- & 57.1 & 22.7 & 81.7 & 36.0 & 58.5 & 26.7 & 65.8 & 28.5 \\
 & $k=2$ & -- & 44.7 & 45.6 & 75.0 & 72.0 & 35.4 & 53.3 & 51.7 & 57.0 \\
 & $k=4$ & -- & 33.8 & 91.2 & 59.3 & 143.5 & 20.1 & 106.3 & 37.7 & 113.7 \\
\midrule
\rowcolor{blue!5} & $B=4$ & -- & \textbf{78.9} & 60.9 & 73.8 & 91.7 & 76.8 & 87.2 & \textbf{76.5} & 79.9 \\
\rowcolor{blue!5} & $B=8$ & -- & 77.5 & 61.1 & 75.3 & 116.1 & 76.2 & 104.3 & \underline{76.3} & 93.8 \\
\rowcolor{blue!5} \multirow{-3}{*}{\textbf{\ours{} (Ours)}} & $B=16$ & -- & 74.9 & 59.7 & 75.0 & 125.2 & 73.8 & 115.7 & 74.5 & 100.2 \\
\bottomrule
\end{tabular}
\end{table}

\Cref{fig:throughput} shows end-to-end generation throughput measured on a single H100 GPU with batch size one with \texttt{bfloat16} inference and a 32-token warm-up.
We report tokens per second (TPS), which is the total generated tokens divided by total generation time.
\Cref{tab:throughput_full} reports the per-task values underlying \Cref{fig:throughput}, and method-specific settings are given below.

\begin{itemize}[leftmargin=*]
    \item \textbf{AR SFT.} Qwen2.5 7B AR SFT uses standard HuggingFace greedy generation and emits one new token per decoding forward.
    \item \textbf{\ours{} (Ours).} For each block size $B\in\{4,8,16\}$, we use the corresponding trained checkpoint with self-speculative decoding.
    A one-step denoiser draft of the next block is verified against the model's own greedy AR predictions.
    The fused pass packs the clean draft with $B$ noisy successor blocks, one for each possible accepted-prefix length.
    Clean draft positions attend causally to produce the verification logits, while each successor block attends to the prefix and its candidate accepted draft prefix and bidirectionally within the block.
    After selecting the longest verified prefix, we retain its KV states and use the corresponding successor block as the next draft, avoiding a separate redrafting forward.
    \item \textbf{DeepSeek-MTP.}
    We train 7B scale DeepSeek-MTP with one DeepSeek-style auxiliary multi-token-prediction stage ($D=1$~\citep{deepseekv3}).
    Greedy verified self-speculation compares the multi-token proposal against the base model's greedy AR prediction.
    \item \textbf{MARS.}
    We use the released MARS Qwen2.5 7B checkpoints~\citep{jin2026mars} for $B\in\{4,8,16\}$.
    We report results using the paper's probability-threshold decoder at threshold $0.95$, with its sliding KV cache.
    Each round commits the leftmost proposal, then commits the longest consecutive continuation whose top-token probabilities remain at least the confidence threshold.
    \item \textbf{Fast-dLLM~v2.}
    We use the released Fast-dLLM~v2 7B checkpoint~\citep{wu2026fast} with its native confidence decoder at threshold $0.9$.
    Block caching is disabled, and the global and sub-block sizes are both set to $B\in\{4,8,16\}$.
    \item \textbf{Efficient-DLM.}
    We evaluate Efficient-DLM 8B~\citep{fu2026edlm} with its native confidence-threshold decoder at threshold $0.7$ for $B\in\{4,8,16,32\}$, together with the $B=32$ speed--quality sweep at confidence thresholds in $\{0.5,0.6\}$.
    We use the model's native prompt format over Qwen chat templates.
    \item \textbf{Dream.}
    We use Dream-v0 7B Instruct~\citep{dream2024} with its native entropy-based unmasking, targeting a mean of $k\in\{1,2,4\}$ filled canvas positions per denoising forward.
    For throughput, we follow Efficient-DLM~\citep{fu2026edlm} and compute TPS as the total number of generated canvas tokens divided by total generation time.
\end{itemize}

\subsection{Confidence decoding}
\label{app:confidence_results}

\begin{figure}[t!]
\centering
\includegraphics[width=0.6\linewidth]{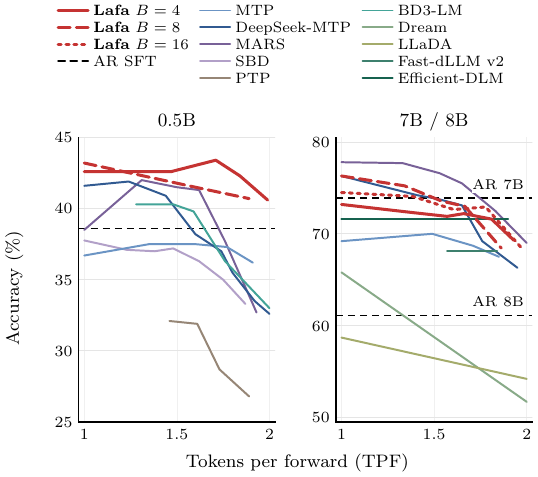}
\caption{Confidence decoding, averaged over GSM8K, HumanEval, and IFEval.
In the right panel, \ours{} uses 8B for $B=4$ and 7B for $B\in\{8,16\}$.
Efficient-DLM ($B=32$) and LLaDA are 8B, and other baselines are 7B.}
\label{fig:confidence_frontier}
\end{figure}

As an alternative to self-speculative decoding, we follow MARS~\citep{jin2026mars} and accept the longest left-to-right prefix of a one-step flow draft whose token confidence exceeds a threshold $\lambda$ (\Cref{app:denoiser_inference}), and apply the same rule to block-autoregressive baselines.
We implement it with the fused decoding of \Cref{alg:fused-denoiser-inference}, where $\hat{\bf y}^{1:a}$ is instead the longest prefix of $\hat{\bf y}$ with confidences $({\bf p}^i)_{\hat{\bf y}^i}\geq\lambda$ (where we always set $a\geq 1$).
Since $a$ is known before the forward pass, only the block ${\bf z}^{(a)}$ is needed instead of all $B$ blocks.
To apply this strategy, Dream and LLaDA use their native global entropy or confidence ordering over masked positions, whereas Fast-dLLM~v2 applies the threshold within the active sub-block.
In \Cref{fig:confidence_frontier}, we show an accuracy--TPF tradeoff measured by sweeping over confidence thresholds.
At 0.5B, \ours{} with block size four attains the highest peak accuracy among all methods, including AR SFT, and its accuracy remains at or above that of pure AR at high thresholds.
At 7B, \ours{} exceeds AR SFT at high thresholds, while MARS, which is trained for confidence decoding, degrades more gracefully as the threshold is lowered.
At 8B, \ours{} remains above AR SFT across the tested range and on par with Efficient-DLM.

Since AR verification avoids the degradation at lower thresholds without tuning a threshold, we use self-speculative decoding as our default in the main text.

\end{document}